\documentclass[sigconf]{acmart}
\AtBeginDocument{%
  }

\setcopyright{acmlicensed}
\copyrightyear{2018}
\acmYear{2018}
\acmDOI{XXXXXXX.XXXXXXX}
\acmISBN{978-1-4503-XXXX-X/2018/06}
\theoremstyle{plain}
\newtheorem{theorem}{Theorem}[section]
\newtheorem{proposition}[theorem]{Proposition}

\theoremstyle{definition}
\newtheorem{definition}[theorem]{Definition}

\theoremstyle{remark}
\newtheorem{remark}[theorem]{Remark}

\newcommand{\R}{\mathbb{R}}
\newcommand{\W}{\mathcal{W}}
\newcommand{\diag}{\mathrm{diag}}
\newcommand{\tr}{\mathrm{tr}}
\newcommand{\eps}{\varepsilon}
\usepackage{color}
\usepackage{algorithm}
\usepackage{algorithmic}
\usepackage{graphicx}  % 导入图像包
\usepackage{subfig}    % 导入子图包

\graphicspath{{images/}} 
\begin{document}

%%
%% The "title" command has an optional parameter,
%% allowing the author to define a "short title" to be used in page headers.
\title{A Decomposition Framework for Certifiably Optimal Orthogonal Sparse PCA}

%%
%% The "author" command and its associated commands are used to define
%% the authors and their affiliations.
%% Of note is the shared affiliation of the first two authors, and the
%% "authornote" and "authornotemark" commands
%% used to denote shared contribution to the research.
\author{Difei Cheng}
\email{chengdifei@amss.ac.cn}
\affiliation{%
  \institution{School of Computer Science and Artificial Intelligence\\
       Aerospace Information Technology University}
  \city{Jinan}
  \state{Shandong}
  \country{China}
}

\author{Qiao Hu}
\email{huqiao2020@amss.ac.cn}
\affiliation{%
  \institution{Academy of Mathematics and Systems Science\\ 
  Chinese Academy of Sciences,}
  \city{Beijing}
  \country{China}
}
\authornote{Corresponding Author}

%%
%% By default, the full list of authors will be used in the page
%% headers. Often, this list is too long, and will overlap
%% other information printed in the page headers. This command allows
%% the author to define a more concise list
%% of authors' names for this purpose.
% \renewcommand{\shortauthors}{Trovato et al.}

%%
%% The abstract is a short summary of the work to be presented in the
%% article.
\begin{abstract}
  Sparse Principal Component Analysis (SPCA) is an important technique for high-dimensional data analysis, improving interpretability by imposing sparsity on principal components. However, existing methods often fail to simultaneously guarantee sparsity, orthogonality, and optimality of the principal components. 
  To address this challenge, this work introduces a novel Sparse Principal Component Analysis (SPCA) algorithm called \textsc{GS-SPCA} (SPCA with Gram-Schmidt Orthogonalization), which simultaneously enforces sparsity, orthogonality, and optimality. However, the original GS-SPCA algorithm is computationally expensive due to the inherent $\ell_0$-norm constraint. To address this issue, we propose two acceleration strategies:  
  First, we combine \textbf{Branch-and-Bound} with the GS-SPCA algorithm. By incorporating this strategy, we are able to obtain $\varepsilon$-optimal solutions with a trade-off between precision and efficiency, significantly improving computational speed.   
  Second, we propose a \textbf{decomposition framework} for efficiently solving \textbf{multiple} principal components. This framework approximates the covariance matrix using a block-diagonal matrix through a thresholding method, reducing the original SPCA problem to a set of block-wise subproblems on approximately block-diagonal matrices.
  % Second, We propose a \textbf{decomposition framework} for efficiently solving \textbf{multiple} principal components. For block-diagonal covariance matrices, the framework decomposes the SPCA problem for solving any sparse principal component into independent block-wise subproblems. For general covariance matrices, the framework can be extended through a thresholding-based block-diagonal approximation scheme, reducing the original SPCA problem to a set of subproblems on approximately block-diagonal matrices.

  % Experimental results demonstrate that combining these acceleration strategies with \textsc{GS-SPCA} significantly enhances computational efficiency.

\end{abstract}

%%
%% The code below is generated by the tool at http://dl.acm.org/ccs.cfm.
%% Please copy and paste the code instead of the example below.
%%
% \begin{CCSXML}
% <ccs2012>
%  <concept>
%   <concept_id>00000000.0000000.0000000</concept_id>
%   <concept_desc>Do Not Use This Code, Generate the Correct Terms for Your Paper</concept_desc>
%   <concept_significance>500</concept_significance>
%  </concept>
%  <concept>
%   <concept_id>00000000.00000000.00000000</concept_id>
%   <concept_desc>Do Not Use This Code, Generate the Correct Terms for Your Paper</concept_desc>
%   <concept_significance>300</concept_significance>
%  </concept>
%  <concept>
%   <concept_id>00000000.00000000.00000000</concept_id>
%   <concept_desc>Do Not Use This Code, Generate the Correct Terms for Your Paper</concept_desc>
%   <concept_significance>100</concept_significance>
%  </concept>
%  <concept>
%   <concept_id>00000000.00000000.00000000</concept_id>
%   <concept_desc>Do Not Use This Code, Generate the Correct Terms for Your Paper</concept_desc>
%   <concept_significance>100</concept_significance>
%  </concept>
% </ccs2012>
% \end{CCSXML}

% \ccsdesc[500]{Do Not Use This Code~Generate the Correct Terms for Your Paper}
% \ccsdesc[300]{Do Not Use This Code~Generate the Correct Terms for Your Paper}
% \ccsdesc{Do Not Use This Code~Generate the Correct Terms for Your Paper}
% \ccsdesc[100]{Do Not Use This Code~Generate the Correct Terms for Your Paper}

%%
%% Keywords. The author(s) should pick words that accurately describe
%% the work being presented. Separate the keywords with commas.
%% A "teaser" image appears between the author and affiliation
%% information and the body of the document, and typically spans the
%% page.
% \begin{teaserfigure}
%   \includegraphics[width=\textwidth]{sampleteaser}
%   \caption{Seattle Mariners at Spring Training, 2010.}
%   \Description{Enjoying the baseball game from the third-base
%   seats. Ichiro Suzuki preparing to bat.}
%   \label{fig:teaser}
% \end{teaserfigure}

% \received{20 February 2007}
% \received[revised]{12 March 2009}
% \received[accepted]{5 June 2009}

%%
%% This command processes the author and affiliation and title
%% information and builds the first part of the formatted document.
\maketitle
\section{Introduction}
Principal component analysis (PCA) \cite{jolliffe2011principal} is a fundamental technique for multivariate data analysis. By transforming the coordinate system to an orthonormal basis aligned with the directions of maximal variance, PCA facilitates dimensionality reduction and is widely applied in various machine learning applications (e.g., clustering \cite{friedman2009elements} and classification \cite{cheng2023fast}).

Despite its success, classical PCA has well-known limitations in high-dimensional settings. Principal loading vectors are often \emph{dense}, with a single component assigning nonzero weight to nearly every variable, which reduces interpretability and complicates variable selection \cite{zou2006sparse}.

Sparse principal component analysis (SPCA) \cite{zou2018selective} addresses these challenges by imposing a sparsity constraint on the loading vectors, i.e., to reduce the $\ell_0$ norm of each vector. Sparse PCA, like PCA, identifies a linear combination of features that captures most of the variance as principal components, but it restricts each component to use only a small subset of features. This sparsity improves interpretability, enables implicit variable selection, and reveals more meaningful structure in high-dimensional data. SPCA has thus been widely adopted in fields where localization and interpretability are crucial, including genomics \cite{ma2011principal,lee2012sparse,hsu2014sparse}, neuroscience \cite{zhuang2020technical}, and text analysis \cite{zhang2011large}.

The price of interpretability is algorithmic difficulty.
SPCA is an NP-hard problem \cite{magdon2017np}, and exact methods require exponential time, as do most approximate methods in the worst case (e.g., mixed-integer optimization and branch-and-bound \cite{berk2019certifiably}).
A large literature has developed a spectrum of approaches--convex relaxations \cite{jolliffe2003modified, witten2009penalized, zou2006sparse, leng2009general, d2004direct}, greedy/thresholding \cite{d2008optimal,moghaddam2005spectral, amini2008high, journee2010generalized, hein2010inverse}, and
approximation algorithms\cite{chan2015worst, beck2016sparse, bertsimas2022solving,li2020exact, dey2023solving, dey2022using,chowdhury2020approximation, del2023sparse, papailiopoulos2013sparse,cory2022sparse,asteris2015sparse}--each trading off statistical fidelity, optimality guarantees, and scalability.
This trade-off is particularly acute when one seeks \emph{certifiable} optimality rather than merely good feasible
solutions.

While the majority of SPCA research focuses on computing a \emph{single} leading sparse component, many real uses of
PCA/SPCA require \emph{multiple} components (second, third, and beyond).
Multi-component structure is essential whenever the goal is to learn a low-dimensional \emph{subspace} rather than a
single direction: dimension reduction to $r>1$ for visualization \cite{ivosev2008dimensionality}, subspace-based clustering \cite{mcwilliams2014subspace} and representation learning \cite{bengio2013representation},
and as a preprocessing stage intended to mitigate multicollinearity in subsequent models.
Accordingly, understanding and computing several sparse principal components is not a minor extension—it is the
setting that most closely matches how PCA is actually used.

However, extending SPCA from one component to many is subtle.
In classical PCA, sequential deflation/projection yields a sequence of orthogonal components and is equivalent to
optimal subspace recovery: the first $r$ components jointly maximize explained variance among all $r$-dimensional
orthonormal systems.
In SPCA, by contrast, it is hard to maintain the same triad simultaneously: (i) sparsity of each
component, (ii) mutual orthogonality among components, and (iii) global optimality for the remaining variance in the
same sense as PCA.
As a consequence, existing multi-component SPCA procedures often relax at least one of these desiderata.
For example, a natural baseline is to compute $x_1$ by solving a single-component SPCA instance and then apply the
same solver to an \emph{adjusted covariance matrix} obtained by projecting out previous directions \cite{berk2019certifiably}.
This strategy is attractive computationally, but it does not guarantee orthogonality of the resulting sparse
components.

Lack of orthogonality can be problematic in practice. Non-orthogonal sparse components may become redundant by capturing overlapping directions, which can inflate the explained variance and reintroduce multicollinearity—an issue that PCA-based preprocessing is specifically designed to address.

These issues motivate algorithms that enforce orthogonality explicitly rather than treating it as an emergent property.

In this work, we propose an orthogonality-preserving algorithm to solve SPCA (Sparse Principal Component Analysis) for multiple principal components, which simultaneously targets sparsity, orthogonality, and (certifiable) optimality. The core of our algorithm integrates a local orthogonalization mechanism based on Gram–Schmidt orthogonalization, ensuring that the computed sparse components remain mutually orthogonal throughout the procedure.

A second challenge is scalability.
Even for one component, SPCA can be computationally demanding; for multiple components the burden is amplified.
Recent work introduced a powerful acceleration paradigm for the \emph{first principal component} case: by constructing a
block-diagonal approximation of the covariance matrix, the SPCA solution can be recovered by solving a
collection of lower-dimensional block-wise subproblems and selecting the best candidate \cite{delefficient}.
This block-diagonalization viewpoint yields dramatic speedups and can be combined in a plug-and-play fashion with
off-the-shelf SPCA solvers.

We extend this block-diagonalization paradigm from one component to many.
Concretely, we prove that for a block-diagonal matrix input, if the first few sparse principal components are
computed in a block-wise manner, then subsequent sparse principal components can \emph{also} be computed block-wise
without loss in objective value.
We further extend this method consistency to the $\varepsilon$-optimal setting, providing a theoretical
justification for combining our framework with branch-and-bound strategies that deliver $\varepsilon$-certificates.
By integrating these results with our orthogonality-preserving multi-component solver, we obtain a practical method
that is both principled and significantly faster in regimes where block structure (exact or approximate) is present.

Our main \textbf{contributions} are summarized as follows:
\begin{enumerate}
    \item \textbf{The first certifiably optimal algorithm with strict sparsity and orthogonality (in subsection \ref{sec:spca-mio}).}
    We propose GS-SPCA, a novel algorithm that is to compute multiple sparse principal components for the $\ell_0$-constraint SPCA problem.  To the best of our knowledge, GS-SPCA is the first algorithm which simultaneously enforces \emph{exact $\ell_0$-sparsity} and \emph{strict orthogonality} across all sparse principal components. This is achieved by integrating a Gram--Schmidt procedure within a combinatorial search, ensuring that the solution satisfies the rigorous definition of orthogonal sparse PCA (Definition \ref{def:ospca}).

    \item \textbf{Integration with branch-and-bound for $\varepsilon$-optimal solutions (in subsection \ref{sec:branch-and-bound}).}
    Our \textsc{GS-SPCA} framework can be embedded into mixed-integer optimization (MIO) solvers. By incorporating the Gram–Schmidt procedure into the branch-and-bound framework, we obtain an accelerated Algorithm~\ref{alg:branch-and-bound} that delivers an $\varepsilon$-optimal solution while maintaining rigorous orthogonality and sparsity. This integration ensures a practical balance between solution quality and computational effort, keeping run-time reasonable.
     
     \item \textbf{A provable decomposition theorem for block-diagonal matrices (in Section \ref{sec:decomposition_theorem_block_diagonal_spca}).}
     For matrices with a block-diagonal structure, we prove two fundamental decomposition theorems (Theorems \ref{thm:block-decomp} and \ref{thm:eps-block-decomp}). These theorems provide an efficient computational framework and show that the original high‑dimensional SPCA problem can be decomposed into a set of independent, much smaller SPCA subproblems on blocks. In other words, merging the exact solutions (or $\epsilon$-optimal solutions) of these block SPCA subproblems results in the exact (or $\epsilon$-optimal) global solution of the block-diagonal SPCA problem.

    \item \textbf{A efficient and provable decomposition framework on general matrices (in Section\ref{sec:spca_thresholding_block_diagonalization}).}
    Leveraging the decomposition theorem, we develop a scalable framework applicable to general matrices. It first approximates the input covariance matrix by a block-diagonal matrix via thresholding, then solves the decomposed block-wise subproblems, and finally recovers a high-quality solution for the original problem.
    % with a provable approximation guarantee $(2p\delta + \varepsilon)$ (Theorem \ref{thm:threshold-alg-correctness}). 
\end{enumerate}

\section{Notation}\label{sec:notation}
Let $Q\in\R^{n\times n}$ be a symmetric positive semi-definite covariance matrix. For a vector $x\in\R^n$, we use $\|x\|_0$ for the $\ell_0$-norm (i.e., number of non-zero elements) and $\|x\|_2$ for the Euclidean norm. The identity matrix of size $n$ is written as $I_n$ and the trace of a matrix $Q$ as $\tr(Q)$.

When the covariance matrix $Q$ has a block-diagonal structure, we write $Q=\diag(Q_1,\dots,Q_d)$, where each $Q_i\in\R^{n_i\times n_i}$  is a symmetric positive semi-definite block with $\sum_{i=1}^d n_i = n$. Let $I_i \subset \{1,\dots,n\}$ denote the coordinate index set associated with block $Q_i$. Correspondingly, the whole space $\R^n$ decomposes into a direct sum of subspaces
\[
\R^n = \W_1 \oplus \dots \oplus \W_d,
\]
where each $\W_i \subset \R^n$ consists of vectors whose support is confined to the coordinate index of $Q_i$, i.e.,
\[
\mathcal W_i
=
\{x \in \mathbb{R}^n : x_j = 0 \ \text{for all } j \notin I_i\}.
\]For any vector $y^{(\W_i)}\in\R^n$ in subspace $\W_i$, we denote by
$y^{(Q_i)} \in \R^{n_i}$ its restriction to the coordinate index of $Q_i$. In other words, $y^{(Q_i)}$ is the restriction of $y^{(\W_i)}$ to the indices in $I_i$, while $y^{(\W_i)}$ is the zero-padding of $y^{(Q_i)}$ to the full $n$-dimensional space. %\sout{Thus, the two representations are equivalent up to this padding/restriction operation.}
%\textcolor{red}{whose entries indexed by $I_i$ equal $y^{(Q_i)}$ and are zero elsewhere. }

\section{Sparse Principal Component Analysis}

Given a symmetric positive semi-definite covariance matrix $Q \in \mathbb{R}^{n \times n}$, Principal Component Analysis (PCA) seeks a sequence of orthogonal directions that
successively maximize the variance $x^\top Qx$. The first principal component is obtained by solving
\[
\max_{\|x\|_2 = 1} \; x^\top Q x.
\]
Subsequent components are defined recursively by imposing orthogonality to all previously extracted components. Specifically, for $k \ge 2$, given the first $(k-1)$ principal
components $x_1, \dots, x_{k-1}$, then the $k$-th principal component $x_k$ is defined recursively by solving \emph{the $k$-th PCA subproblem}
\begin{equation}
\label{eq:pca}
\begin{aligned}
\max \quad & x^\top Q x, \\
\mathrm{subject~to} \quad & \|x\|_2 = 1, \\
& x \perp x_1,\dots,x_{k-1}.
\end{aligned}
\end{equation}

%Sparse PCA (SPCA) extends PCA by imposing a cardinality constraint on the loading vectors. 
Sparse Principal Component Analysis (SPCA) introduces sparsity into classical PCA by imposing sparsity constraints directly to the principal components, typically through the classic cardinality constraint (i.e. $\ell_0$)  or other sparsity-inducing methods. 
In this work, we focus on a formulation of SPCA that imposes the classic $\ell_0$-constraint explicitly on the coefficients of the principal components, ensuring that only a subset of the original features contribute significantly to each component. The first sparse principal component is defined as a solution to
the following optimization problem:
\begin{equation}
\label{eq:spca1}
\begin{aligned}
\max \quad & x^\top Q x, \\
\mathrm{subject~to}  \quad & \|x\|_2 = 1, \\
& \|x\|_0 \leq p,
\end{aligned}
\end{equation}
where the integer $p\leq n$ is a prescribed sparsity parameter. Following the classical PCA, we define subsequent sparse principal components by simultaneously enforcing sparsity and orthogonality constraints.
\begin{definition}[Orthogonal Sparse Principal Components]\label{def:ospca}
For a symmetric positive semi-definite matrix $Q \in \mathbb{R}^{n \times n}$ and a sparsity parameter $p$, the first sparse principal component $x_1$ is defined as any solution to problem \eqref{eq:spca1}. Then, for $2\leq k \leq n$, the $k$-th sparse principal component $x_k$ is defined sequentially as any solution to \emph{the $k$-th SPCA subproblem}
\begin{equation}
\label{eq:ospca}
\begin{aligned}
\max \quad & x^\top Q x, \\
\text{s.t.} \quad & \|x\|_2 = 1, \\
& \|x\|_0 \leq p, \\
& x \perp x_1, x_2, \dots, x_{k-1}.
\end{aligned}
\end{equation}
\end{definition}
This definition ensures that each sparse principal component is orthogonal to all previous ones, preserving the geometric structure of classical PCA. By construction, the obtained variance sequence $\{x_k^\top Q x_k\}_{k=1}^n$ is non-increasing.

\textbf{SPCA problem:} Given an $n$-dimensional symmetric positive semi-definite matrix $Q$, the SPCA problem is to find a set of vectors $x_1, \dots, x_n$ that satisfy Definition \ref{def:ospca}.

It is well known that, in classical PCA, the variance sequence $\{x_k^\top Q x_k\}_{k=1}^n$ is unique and equals the eigenvalues of $Q$. But in SPCA, this sequence is generally not unique. We refer to this phenomenon as \emph{the Path Dependency of Variance in SPCA} and illustrate it with an example in Section \ref{sec:future-work}.
Nevertheless, despite this non-uniqueness at the level of individual components,
all SPCA solutions share an \textbf{ global invariant}: the total variance
$\sum_{k=1}^n x_k^\top Q x_k$ always equals $\operatorname{tr}(Q)$,
as established in Proposition~\ref{prop:tr_spca}.

\begin{proposition}\label{prop:tr_spca}
Let $\{x_1, x_2, \dots, x_n\}$ be a solution to the SPCA problem, i.e., a set of vectors satisfying Definition \ref{def:ospca}. Then, these components form a complete orthonormal basis of $\mathbb{R}^n$ and satisfy the equality
\[
\sum_{k=1}^n x_k^\top Q x_k = \tr(Q).
\]
Moreover, the variance sequence $\{x_k^\top Q x_k\}_{k=1}^n$ is non-increasing.
\end{proposition}

\begin{proof}
Since $\{x_1, \dots, x_n\}$ are $n$ orthonormal unit vectors in $\R^n$, they constitute an orthonormal basis of $\mathbb{R}^n$. Let $X=[x_1, \dots, x_n]\in\R^{n\times n}$ denote the orthogonal matrix. Then,
\[
\sum_{k=1}^n x_k^\top Q x_k = \tr(X^\top Q X) = \tr(Q X X^\top) = \tr(Q).
\]
Moreover, for any $k\geq 1$, the $(k+1)$-th sparse principal component $x_{k+1}$ is a feasible solution of the $k$-th SPCA subproblem (i.e., Problem \eqref{eq:spca1} for $k=1$ and Problem \eqref{eq:ospca} for $k\geq 2$). Thus,
\[
x_k^\top Q x_k \geq x_{k+1}^\top Q x_{k+1}, \quad \forall 1\leq k\leq n-1.
\]
\end{proof}

\begin{remark}\label{rmk:tr}
Existing literature \cite{berk2019certifiably} employs an adjusted covariance matrix to compute subsequent sparse principal components approximately. Specifically, after obtaining $x_1$, one constructs the matrix 
\[Q_2 = (I - x_1 x_1^\top) Q (I - x_1 x_1^\top)
\] 
and solves \eqref{eq:spca1} with $Q_2$ to obtain $x_2$. Subsequent components are derived analogously by repeating this process. Under this solution approach, the resulting components are not guaranteed to be
orthogonal. Consequently, the equality in Proposition~\ref{prop:tr_spca} does not
hold in general, rendering cumulative variance-based model selection criteria
ill-posed.

%专业版：从而使得基于累计解释方差的模型选择准则变得不适定；说人话：由于所得主成分不再正交，其对应的方差贡献无法通过简单求和来刻画，因此也无法通过“累计解释方差占比达到某一阈值（例如 99%）”这一标准来合理地确定主成分的个数

\end{remark}

\section{A Novel Algorithm for Orthogonal SPCA}\label{sec:novel_algorithm_orthogonal_spca}

%The decomposition theorems presented in the previous sections reduce the SPCA problem for block-diagonal matrices to solving smaller subproblems on each block. 
 
The orthogonal SPCA problem remains challenging due to the combinatorial $\ell_0$-constraint and the orthogonality requirements. 
Existing methods for SPCA often relax one or both of the constraints. A common approach is to replace the $\ell_0$-norm with an $\ell_1$-norm penalty, which leads to convex relaxations but does not guarantee exact sparsity or orthogonality. Another popular technique is the deflation method \cite{berk2019certifiably}, which iteratively computes sparse principal components by adjusting the covariance matrix, $Q_{k+1} = (I - x_k x_k^\top) Q (I - x_k x_k^\top)$. This method also formulates SPCA as a mixed integer optimization (MIO) problem, referred to as SPCA-MIO, and solves it to certifiable optimality using a branch-and-bound algorithm. However, as previously noted in Remark \ref{rmk:tr}, this method does not explicitly enforce orthogonality, and as a result, the computed components are generally not orthogonal, as demonstrated by the experimental results in Figure \ref{fig:comparison} (a-c).

In this section, we propose a new algorithm to solve the SPCA-MIO problem but enforcing strict orthogonality among all sparse principle components. Our method combines the SPCA-MIO model with a Gram–Schmidt orthogonalization step to ensure that each new component is orthogonal to all previous components, thereby satisfying Definition \ref{def:ospca}. We describe the algorithm in detail as follows.

\subsection{SPCA-MIO with Gram--Schmidt Orthogonalization}\label{sec:spca-mio}
The following $k$-th SPCA-MIO formulation is equivalent to the $\ell_0$-constrained  SPCA subproblem
\eqref{eq:ospca}:
\begin{equation}
\label{eq:spca-mio}
\begin{aligned}
\max \quad & x^\top Q x, \\
\mathrm{s.t.}  \quad & x \perp x_1,\dots,x_{k-1},\\
& \|x\|_2 = 1, \\
& \|y\|_0 = p,\\
& -y[j]\le x[j] \le y[j],\;\; 1\le j\le n,\;\;\text{(SPCA-MIO)}\\
& x\in[-1,1]^n,\\
& y\in\{0,1\}^n,
\end{aligned}
\end{equation}
where $x[j]$ and $y[j]$ represent the $j$-th elements of vectors $x$ and $y$, respectively.
For a fixed sparsity level $p$, we propose an algorithm to compute the vector $x_k$ for the $k$-th SPCA-MIO problem in \eqref{eq:spca-mio}. The key idea is to enumerate all possible support sets of vector $y$ for each component, solve a reduced PCA problem on that support under the orthogonality constraints, and then select a principle component that yields the maximum variance. 

In specific, for any candidate support vector $y$, let $Y=\{j:y[j]=1\} \subseteq \{1,\dots,n\}$ be its support set with $|Y| = p$. Denote by $Q_Y \in \mathbb{R}^{p \times p}$ the principal submatrix of $Q$ corresponding to the indices in $Y$, i.e., $Q_Y = Q[Y,Y]$. For each previously computed component $x_i$ ($1 \le i \le k-1$), let $v_i = x_i[Y] \in \mathbb{R}^p$ be the restriction of $x_i$ to the indices in $Y$. Note that the vectors $\{v_1,\dots,v_{k-1}\}$ are not necessarily orthogonal in $\mathbb{R}^p$. Let $U_Y \in \mathbb{R}^{p \times m}$ be an orthonormal basis of the subspace spanned by $\{v_1,\dots,v_{k-1}\}$, which can be obtained by applying the Gram--Schmidt process (here $m \le \min(p, k-1)$ is the dimension of the spanned subspace). 

Now we only need to solve a reduced PCA problem on the support $Y$:
\begin{equation}
\label{eq:reduced-pca}
\begin{aligned}
\max_{\|z\|_2 = 1,\;z \in \mathbb{R}^p} \quad & z^\top (I-U_YU_Y^\top)Q_Y(I-U_YU_Y^\top) z.
\end{aligned}
\end{equation}
Thus, for each candidate support $Y$, we can compute the best unit vector $z$ supported on $Y$ that is orthogonal to all previous components. The optimal value is the largest eigenvalue $\lambda_{\max}(P_Y Q_Y P_Y)$, where $P_Y = I-U_YU_Y^\top$. The global $k$-th sparse principal component is then obtained by  maximizing $\lambda_{\max}(P_Y Q_Y P_Y)$ among all $Y$. The complete procedure for computing the $k$-th sparse principal component is summarized in Algorithm \ref{alg:ksparse} (called GS-SPCA). To obtain a full set of $n$ orthogonal sparse principal components, we simply run Algorithm \ref{alg:ksparse} sequentially for $k=1,\dots,n$ (for $k=1$, there are no orthogonality constraints, so $P_Y = I$).

\begin{remark} If $m=p$, then $P_Y=I-U_YU_Y^\top=0$, and consequently $\lambda_{\max}(P_Y Q_Y P_Y)=0$. In this case, the reduced problem \eqref{eq:reduced-pca} is degenerate, in the sense that all unit vectors attain the same objective value $0$. This is mathematically self-consistent and does not require any special treatment in the algorithm. \end{remark}

\begin{algorithm}[t]
\caption{Solving the $k$-th SPCA-MIO with Gram--Schmidt Orthogonalization (GS-SPCA)}
\label{alg:ksparse}
\begin{algorithmic}[1]
\REQUIRE Symmetric matrix $Q \in \mathbb{R}^{n \times n}$, sparsity $p$, a matrix formed by previous components $X_k=[x_1,\dots,x_{k-1}]$ (empty if $k=1$).
\ENSURE A $k$-th sparse principal component $x_k$.
\STATE Initialize $var \gets -\infty$, $x_k \gets \{0\}^n$.
\FOR{each subset $Y \subseteq \{1,\dots,n\}$ with $|Y| = p$}
    \STATE Extract $Q_Y = Q[Y,Y]$ and $v_i = x_i[Y]$ for $i=1,\dots,k-1$.
    \STATE Apply Gram--Schmidt process to generate an orthonormal basis $U_Y$ of $\operatorname{span}\{v_1,\dots,v_{k-1}\}$.
    \STATE Let $P_Y = I - U_YU_Y^\top$ (if $k=1$, $P_Y = I$).
    \STATE Compute the largest eigenvalue $\lambda$ and a corresponding unit eigenvector $z$ of $P_Y Q_Y P_Y$.
    \IF{$\lambda > var$}
        \STATE $var \gets \lambda$, $z_{\max} \gets z$, $Y_{\max} \gets Y$.
    \ENDIF
\ENDFOR
\STATE Set $x_k[Y_{\max}] = z_{\max}$ and $x_k[j] = 0$ for $j \notin Y_{\max}$.
\STATE \textbf{Return} $x_k$.
\end{algorithmic}
\end{algorithm}

\begin{theorem}
For any symmetric matrix $Q \in \mathbb{R}^{n \times n}$, sparsity $p \le n$, and given previous components $x_1, \dots, x_{k-1}$ (empty when $k=1$), Algorithm \ref{alg:ksparse} returns a vector $x_k$ that is a global optimal solution to the $k$-th SPCA subproblem \eqref{eq:ospca}. Moreover, the algorithm terminates after evaluating exactly $\binom{n}{p}$ subsets.
\end{theorem}

\begin{proof}
Algorithm \ref{alg:ksparse} enumerates every subset $Y \subseteq \{1,\dots,n\}$ with $|Y|=p$. For each such $Y$, it computes the principal submatrix $Q_Y = Q[Y,Y]$ and the projections $v_i = x_i[Y]$ of the previous components. Let $U_Y$ be an orthonormal basis for $\operatorname{span}\{v_1,\dots,v_{k-1}\}$ and $P_Y = I - U_YU_Y^\top$ (with $P_Y=I$ when $k=1$). The algorithm then solves
\[
\max_{\|z\|_2=1} z^\top (P_Y Q_Y P_Y) z,
\]
whose optimal value is the largest eigenvalue $\lambda_{\max}(P_Y Q_Y P_Y)$ and an optimal solution is any corresponding unit eigenvector $z$.

For any feasible $x$ of the $k$-th subproblem \eqref{eq:ospca} with support $Y$, the restricted vector $z = x[Y]$ satisfies $\|z\|_2=1$ and 
\[
\langle z,v_i\rangle = \langle x,x_i\rangle=0,\;\;\forall 1\le i\le k-1,
\]
hence $z$ is feasible for the reduced PCA problem \eqref{eq:reduced-pca} and $$x^\top Q x = z^\top Q_Y z = z^\top (P_Y Q_Y P_Y) z \le \lambda_{\max}(P_Y Q_Y P_Y).$$

Since the algorithm evaluates all $\binom{n}{p}$ possible supports and selects the one with the largest $\lambda_{\max}(P_Y Q_Y P_Y)$, it returns a global optimal solution. The number of subsets is finite, so the algorithm terminates after a finite number of steps.
\end{proof}

\subsection{Accelerative Algorithm via Branch-and-Bound}\label{sec:branch-and-bound}

However, finding exact optimal solutions for these SPCA subproblems can be computationally expensive. Furthermore, the exhaustive enumeration used in Algorithm~\ref{alg:ksparse} becomes impractical when dealing with large-scale SPCA problems. In practice, we can speed up the search process by integrating a branch-and-bound framework.

The core idea is to systematically explore the space of potential supports, while pruning branches that cannot yield a better solution than the current best one. In most cases, it is sufficient to compute approximate solutions that are within a small tolerance $\eps \ge 0$ of the optimal variance.

Similar to Definition \ref{def:ospca}, we define the $\eps$-optimal SPCA solution as follows.

\begin{definition}[$\eps$-optimal SPCA solution]\label{def:eps-spca}
For a given $\eps \ge 0$, a sequence $\{\widetilde{x}_1, \dots, \widetilde{x}_n\}$ is called an $\eps$-optimal solution to the SPCA problem for $(Q,p)$ if it satisfies:
\begin{enumerate}
    \item $\|\widetilde{x}_k\|_2 = 1$, $\|\widetilde{x}_k\|_0 \le p$, and $\widetilde{x}_k \perp \widetilde{x}_1, \dots, \widetilde{x}_{k-1}$ for each $k$.
    \item For each $k = 1, \dots, n$, the variance $\widetilde{x}_k^\top Q \widetilde{x}_k$ is at least $\eps$-close to the optimal value of the $k$-th SPCA subproblem, i.e., 
    \[
    \widetilde{x}_k^\top Q \widetilde{x}_k \ge \max_{\substack{\|x\|_2=1,~\|x\|_0 \le p \\ x \perp \widetilde{x}_1, \dots, \widetilde{x}_{k-1}}} x^\top Q x - \eps.
    \]
\end{enumerate}
\end{definition}

\begin{remark}[Certification of $\eps$-optimality]
\label{rem:eps_certification}
Definition~\ref{def:eps-spca} does not require access to the exact optimal value
of each SPCA subproblem. In practice, $\eps$-optimality can be certified by
comparing a candidate solution against a computable upper bound on the optimal
variance.
For instance, a natural upper bound is given by
the largest eigenvalue of $Q$. Given a candidate solution $\widetilde{x}_k$,
if
\[
\lambda_{\max}(Q) - \widetilde{x}_k^\top Q \widetilde{x}_k \le \eps,
\]
then $\widetilde{x}_k$ is guaranteed to be $\eps$-optimal.
More generally, as the solution process progresses, tighter upper bounds on the
optimal value can be obtained, for example, via
branch-and-bound procedures. As additional information is accumulated, these
upper bounds can be refined to approach the true optimum, enabling increasingly
accurate certification of $\eps$-optimality without solving the subproblem to
exact optimality.
\end{remark}

The accelerative algorithm employs branch-and-bound techniques (as in \cite{berk2019certifiably}) to prune the search space and terminates when the optimality gap is below $\eps$, thereby guaranteeing an $\eps$-optimal solution. Specifically, at each node of the branch-and-bound tree, we compute an upper bound on the maximum achievable variance and a lower bound of variance on a special feasible vector. If this upper bound does not exceed the current best solution value by at least $\eps$, the node is pruned and the subtree rooted at that node never be explored. %When a node corresponds to a fully determined support $Y$ with $|Y|\le p$, we need only solve the reduced PCA problem \eqref{eq:reduced-pca} with the added Gram--Schmidt orthogonalization step. At this node, the upper and lower bound of variance are $\lambda_{\max}(P_YQ_YP_Y)$. 

We refer the reader to Appendix \ref{app:branch-and-bound} for details of the branch-and-bound accelerative algorithm \ref{alg:branch-and-bound}, including the formulation of the lower and upper bounds. This accelerated algorithm allows us to solve larger SPCA problems to $\eps$-optimality within reasonable time.

\section{Decomposition Theorem for Block-Diagonal SPCA}\label{sec:decomposition_theorem_block_diagonal_spca}

Due to the $\ell_0$-constraint condition, Algorithm \ref{alg:ksparse} to solve the SPCA problem \eqref{eq:ospca} (or equivalent SPCA-MIO problem \eqref{eq:spca-mio}) requires enumerating all
$\binom{n}{p}$ candidate support sets in the worst case. However, for covariance matrices with special structure, some efficient strategies can be developed. In this section, we will introduce an important decomposition theorem for solving a block-diagonal SPCA problem, which provides an efficient method to obtain a global SPCA solution from the solutions of the block subproblems.

To present the decomposition theorem, we first introduce  necessary notation for the block-diagonal SPCA. Let $Q = \diag(Q_1, \dots, Q_d)$ be a symmetric block-diagonal matrix. For each block $Q_i\in\R^{n_i\times n_i}$, let $p_i = \min(p, n_i)$ be the corresponding sparsity parameter. Suppose we have obtained a solution to the block SPCA problem for $(Q_i,p_i)$, i.e., a set of vectors 
\[
\{y_1^{(Q_i)}, y_2^{(Q_i)}, \dots, y_{n_i}^{(Q_i)}\} \subset \R^{n_i}
\] 
satisfying Definition \ref{def:ospca} for $(Q,p)=(Q_i,p_i)$. For each such vector $y_j^{(Q_i)}$, we denote its zero-padded extension by $y_j^{(\W_i)}$, as described in Section \ref{sec:notation}.

Now we state two main decomposition theorems.

\begin{theorem}[Decomposition for Block-Diagonal SPCA] \label{thm:block-decomp}
Let $Q = \diag(Q_1, \dots, Q_d)\in\R^{n\times n}$ be a symmetric block-diagonal matrix and $p\le n$ be the sparsity parameter.
For each block $1\le i\le d$, let $\{y_j^{(Q_i)}\}_{j=1}^{n_i}$ be a solution to the block SPCA problem for $(Q_i,p_i)$, and $\{y_j^{(\W_i)}\}_{j=1}^{n_i}$ be its corresponding zero-padded extension. Sort all $n$ extended vectors $\bigcup_{i=1}^d\{y_j^{(\W_i)}\}_{j=1}^{n_i}$ in non-increasing order of their variances $(y_j^{(\W_i)})^\top Q y_j^{(\W_i)}$, and denote the sorted sequence by
\[
\{z_1, z_2, \dots, z_n\}.
\]
Then the sequence $\{z_1, z_2, \dots, z_n\}$ is a valid solution to the SPCA problem for original $(Q,p)$.
\end{theorem}

% \begin{corollary}\label{cor:block-decomp-first}
% For the first sparse principal component, the decomposition is particularly simple:
% \[
% \max_{\|x\|_2=1, \|x\|_0 \leq p} x^\top Q x = \max_{1 \leq i \leq d} \left( \max_{\|x\|_2=1, \|x\|_0 \leq p_i} x^\top Q_i x \right).
% \]
% Moreover, if $x_{i^*}$ is an optimal first sparse principal component for $Q_{i^*}$ achieving the maximum on the right-hand side, then its zero-padded extension $\tilde{x}_{i^*}$ is an optimal first sparse principal component for $Q$.
% \end{corollary}

% This corollary recovers the result of \cite{iclr2025} as a special case of our theorem. The theorem further generalizes to all subsequent sparse principal components, providing a complete decomposition of the SPCA problem for block-diagonal matrices.

%\textcolor{red}{\subsection{Decomposition for Suboptimal Solutions}}
%这里没说具体咋计算suboptimal，怎么降低计算成本，这里需要写个remark说一下，比如说我给了一个epsilon值，我计算出了一个解，怎么判断这个解和optimal的误差是epsilon

\begin{theorem}[Decomposition for $\eps$-optimal Block-Diagonal SPCA]\label{thm:eps-block-decomp}
Let $Q = \diag(Q_1, \dots, Q_d)\in\R^{n\times n}$ be a symmetric block-diagonal matrix and $p\le n$ be the sparsity parameter. For each block $i$, let $\{\widetilde{y}_j^{(Q_i)}\}_{j=1}^{n_i}$ be an $\eps$-optimal solution ($\eps\ge0$) to the block SPCA problem for $(Q_i, p_i)$, and let $\{\widetilde{y}_j^{(\W_i)}\}_{j=1}^{n_i}$ be their zero-padded extensions. Sort the set $\bigcup_{i=1}^d \{\widetilde{y}_j^{(\W_i)}\}_{j=1}^{n_i}$ in non-increasing order of the variances $(\widetilde{y}_j^{(\W_i)})^\top Q \widetilde{y}_j^{(\W_i)}$, and denote the sorted sequence by 
\[
\{\widetilde{z}_1, \dots, \widetilde{z}_n\}.
\]
Then the sequence $\{\widetilde{z}_1, \dots, \widetilde{z}_n\}$ is an $\eps$-optimal solution to the SPCA problem for original $(Q, p)$.
\end{theorem}
\begin{proof}
    See Appendix \ref{app:proof-decop} and \ref{app:proof-decop-eps}, respectively.
\end{proof}

\paragraph{Complexity and practical considerations.} 

Applying SPCA directly to the full matrix $Q$ requires enumerating all
$\binom{n}{p}$ candidate support sets, which becomes computationally
prohibitive when $n$ is large and $p>1$. In contrast, in the block-diagonal
SPCA setting, the matrix $Q$ decomposes into $d$ independent blocks
$\{Q_i\}_{i=1}^d$, where each block $Q_i$ has dimension $n_i$, typically
much smaller than $n$. Consequently, performing SPCA separately on each
$Q_i$ only requires enumerating
\[
\sum_{i=1}^d \binom{n_i}{p_i}
\]
candidate supports, which can be acceptable when the block sizes $n_i$
are small.
In particular, assuming all blocks have the same dimension $n_i = n/d$ and
a parallel implementation is available, this process reduces to enumerating
\[
\binom{n/d}{p}
\]
candidate supports, which is substantially smaller than $\binom{n}{p}$ when $d$ is large. As a result, the proposed approach becomes computationally
feasible for large-scale problems that admit an (approximate)
block-diagonal structure.

%Theorem \ref{thm:eps-block-decomp} has important practical implications. It allows us to use fast approximation algorithms for solving the block SPCA subproblems, provided we can control their approximation guarantees. Moreover, if we apply thresholding to obtain an approximate block-diagonal matrix $Q_\eps$ from a general sparse matrix $Q$, and then solve the SPCA problem for $Q_\eps$ using Theorem \ref{thm:eps-block-decomp}, we obtain an approximate solution for the original matrix $Q$ with a controlled error bound.

\section{General SPCA via Thresholding and Block-Diagonalization}\label{sec:spca_thresholding_block_diagonalization}

In practical scenarios, the covariance matrix $Q$ may not have an exact block-diagonal structure. In this section, we present an approximate method that transforms a general symmetric positive semi-definite matrix into a block-diagonal form via thresholding and graph partitioning \cite{delefficient}. Then we apply the previous decomposition theorems to solve the approximate SPCA problem and provide theoretical guarantees on the quality of the obtained solution.

\subsection{Matrix Block-diagonlization}

Given a symmetric positive semi-definite matrix $Q \in \mathbb{R}^{n \times n}$ and a threshold $\delta \ge 0$, we define the thresholded matrix $Q^\delta$ by setting to zero for all entries whose absolute value is below $\delta$:
\[
Q^\delta_{ij} = 
\begin{cases}
Q_{ij}, & \text{if } |Q_{ij}| \geq \delta, \\
0, & \text{otherwise}.
\end{cases}
\]
The matrix $Q^\delta$ is symmetric and approximates $Q$ in the sense that $\|Q - Q^\delta\|_{\max} \leq \delta$, where $\|\cdot\|_{\max}$ denotes the maximum absolute entrywise difference.

The sparse matrix $Q^\delta$ can be transformed into a block-diagonal matrix by identifying its connected components. Construct an undirected graph $G = (V, E)$ with vertex set $V = \{1, \dots, n\}$ and edge set
\[
E = \{(i,j) : i \neq j \text{ and } (Q^\delta)_{ij} \neq 0\}.
\]
Two vertices are connected if there is a path between them. Let $C_1, C_2, \dots, C_d$ be the connected components of $G$. By permuting the rows and columns of $Q^\delta$ according to these components (i.e., ordering indices so that all indices in $C_1$ come first, then $C_2$, etc.), we obtain a block-diagonal matrix:
\[
A = \Pi Q^\delta \Pi^\top = \diag(A_1, A_2, \dots, A_d),
\]
where $\Pi$ is a permutation matrix, and each block $A_i$ corresponds to the submatrix of $Q^\delta$ restricted to the indices in $C_i$ (see Algorithm \ref{alg:block-diag}).

\begin{algorithm}[t]
\caption{Matrix Block-Diagonalization}
\label{alg:block-diag}
\begin{algorithmic}[1]
\REQUIRE Symmetric matrix $Q \in \mathbb{R}^{n \times n}$, threshold $\delta \geq 0$.
\ENSURE Permutation matrix $\Pi$, block-diagonal matrix $A = \Pi Q^\delta \Pi^\top$.
\STATE \textbf{Thresholding:} Construct $Q^\delta$ by setting entries with $|Q_{ij}| < \delta$ to zero:
\[
Q^\delta_{ij} = 
\begin{cases}
Q_{ij}, & \text{if } |Q_{ij}| \geq \delta, \\
0, & \text{otherwise}.
\end{cases}
\]
\STATE \textbf{Graph construction:} Build undirected graph $G=(V,E)$ with $V=\{1,\dots,n\}$ and 
$E = \{(i,j) : i \neq j \text{ and } Q^\delta_{ij} \neq 0\}$.
\STATE \textbf{Connected components:} Compute the connected components $C_1, C_2, \dots, C_d$ of $G$.
\STATE \textbf{Permutation:} Generate a permutation $\pi$ of $\{1,\dots,n\}$ that lists indices in $C_1$ first, then $C_2$, etc. Let $\Pi$ be the corresponding permutation matrix (i.e., $\Pi_{i,\pi(i)}=1$ for all $i$, others zero).
\STATE \textbf{Block-diagonal form:} Compute $A = \Pi Q^\delta \Pi^\top$. This matrix is block-diagonal with blocks $A_i$ corresponding to components $C_i$.
\STATE \textbf{Return} $\Pi$, $A$. 
\end{algorithmic}
\end{algorithm}

The permutation matrix $\Pi$ is orthogonal ($\Pi^\top = \Pi^{-1}$) and represents the reordering of indices according to the connected components. The block-diagonal structure of $A$ allows us to apply the decomposition theorems (Theorems \ref{thm:block-decomp} and \ref{thm:eps-block-decomp}) independently on each block. Then we have the following theorem to illustrate the relationship between the original SPCA problem with $(Q,p)$ and these block SPCA problems with $(A_i, \min(p, |C_i|))$.

\subsection{Threshold Decomposition for SPCA}

\begin{theorem}\label{thm:threshold-decop}
Given a symmetric positive semi-definite matrix $Q\in\R^{n\times n}$ and sparsity parameter $p\le n$, we obtain the permutation matrix $\Pi$ and block-diagonal matrix $A=\diag(A_1,\dots,A_d)$ by Algorithm \ref{alg:block-diag}. Then for any $\eps$-optimal solution $\{\widetilde{z}_1,\dots,\widetilde{z}_n\}$ to the SPCA problem for $(A,p)$, 
the sequence $\{\Pi^\top\widetilde{z}_1,\dots,\Pi^\top\widetilde{z}_n\}$ is a $(2p\delta+\eps)$-optimal solution to the SPCA problem for original $(Q,p)$.
\end{theorem}
\begin{proof}
    See Appendix \ref{app:proof-approx} for details.
\end{proof}

Theorem \ref{thm:threshold-decop} shows that we can solve the SPCA problem for the approximate matrix $A$ and then map the solution back to the original problem with a quantifiable error bound. However, in practical applications, one often requires only the first $K$ ($K \ll n$) sparse principal components rather than the complete orthonormal basis. The procedure described in Theorem \ref{thm:threshold-decop} can be adapted to compute these first $K$ components efficiently. The idea is to maintain a pool of candidate components, where each candidate corresponds to the next available component from one of the blocks, and iteratively select the candidate with the largest variance until $K$ components are obtained.
%Theorem \ref{thm:threshold-decop} shows that, we can solve the SPCA problem in each block $A_i$ independently, with sparsity parameter $p_i = \min(p, n_i)$. For each block $A_i$, we will obtain an $\eps$-optimal solution $\{\widetilde{y}_1^{(A_i)}, \dots, \widetilde{y}_{n_i}^{(A_i)}\}$ of orthogonal sparse principal components. These vectors will be extended to $\mathbb{R}^n$ by zero-padding outside block $i$, yielding $\{\widetilde{y}_1^{(\W_i)}, \dots, \widetilde{y}_{n_i}^{(\W_i)}\}$. After sorting all extended vectors in descending order of their variances $\widetilde{y}^\top A \widetilde{y}$, we obtain a sorted sequence $\{\widetilde{z}_1, \dots, \widetilde{z}_n\}$. Finally, $\{\Pi^\top\widetilde{z}_1,\dots,\Pi^\top\widetilde{z}_n\}$ is generated by applying the inverse permutation and is a $(2p\delta+\eps)$-optimal solution for the original SPCA problem $(Q,p)$.

Algorithm \ref{alg:threshold-spca} initializes by computing the \emph{first sparse principal component} $\widetilde{y}_1^{(A_i)}$ for each block $A_i$ (line 4-6). The extended vectors form the initial candidate set $S$. The first global component $\widetilde{z}_1$ is simply the vector in $S$ with the maximum variance (line 9). In subsequent iterations ($k \ge 2$), only the block $\ell$ that contributed the previous component needs to update its next component (line 12). This component is computed under the orthogonal constraint with respect to all previously selected components from the same block (ensured by passing the matrix $X_k^{(A_\ell)}$ to Algorithm \ref{alg:ksparse}). The candidate set $S$ is updated by replacing the $\ell$-th candidate with this new vector (line 14), and the $k$-th global component is again the one with the largest variance in $S$ (line 15). After obtaining $\{\widetilde{z}_1, \dots, \widetilde{z}_K\}$, the inverse permutation yields the final solution $\{u_1, \dots, u_K\}$ for the original matrix $Q$.

\begin{algorithm}[t]
\caption{Threshold Decomposition for SPCA}
\label{alg:threshold-spca}
\begin{algorithmic}[1]
\REQUIRE Symmetric matrix $Q \in \mathbb{R}^{n \times n}$, sparsity $p$, threshold $\delta \ge 0$, tolerance $\eps \ge 0$, number of components $1\le K\le n$.
\ENSURE Orthonormal sparse vectors $\{u_1, \dots, u_K\}$, which form a $(2p\delta+\eps)$-optimal solution to the first $K$ SPCA subproblems for $(Q,p)$.
\STATE \textbf{Preprocessing:} Generate permutation matrix $\Pi$ and block-diagonal matrix $A=\diag(A_1,\dots,A_d)$ by Algorithm \ref{alg:block-diag}, where $A_i\in\R^{n_i\times n_i}$ for all $1\le i\le d$.
\FOR {$k=1$ to $K$}
\IF {$k=1$}
\FOR {$i = 1$ to $d$}
    \STATE $p_i \gets \min(p, n_i)$, $X_1^{(A_i)}\gets \emptyset.$
    \STATE $\widetilde{y}_1^{(A_i)}\gets \left\{\begin{aligned} &\text{Algorithm~\ref{alg:ksparse}}(A_i,p_i,X_1^{(A_i)}), \;\; \eps=0\\
    &\text{Algorithm~\ref{alg:branch-and-bound}}(A_i,p_i,X_1^{(A_i)},\eps), \;\; \eps>0
    \end{aligned}\right.$
\ENDFOR
\STATE Collect all corresponding extended vectors of $\R^n$
\[
S \gets \{\widetilde{y}_1^{(\W_1)}, \dots, \widetilde{y}_1^{(\W_d)}\}
\]
\STATE Select the vector $\widetilde{z}_1\in S$ and corresponding block index $\ell$, with the maximum variance among $S$.
\ELSE 
\STATE Update $X_k^{(A_\ell)}\gets[X_{k-1}^{(A_\ell)};\widetilde{y}_{k-1}^{(A_\ell)}]$.
\STATE $\widetilde{y}_k^{(A_\ell)}\gets \left\{\begin{aligned} &\text{Algorithm~\ref{alg:ksparse}}(A_\ell,p_\ell,X_k^{(A_\ell)}), \;\; \eps=0\\
&\text{Algorithm~\ref{alg:branch-and-bound}}(A_\ell,p_\ell,X_k^{(A_\ell)},\eps), \;\; \eps>0
\end{aligned}\right.$
\STATE Update $S$ by replacing the $\ell$-th elements as $\widetilde{y}_k^{(\W_\ell)}$.
\STATE Select the $k$-th component $\widetilde{z}_k\in S$ and update its corresponding block index $\ell$, with the maximum variance among the new set $S$.
\ENDIF
\ENDFOR
\STATE \textbf{Inverse permutation:} $u_k \gets \Pi^\top \widetilde{z}_k$ for $k = 1, \dots, K$.
\STATE \textbf{Return} $u_1, \dots, u_K$.
\end{algorithmic}
\end{algorithm}

\begin{theorem}\label{thm:threshold-alg-correctness}
The output $\{u_1, \dots, u_K\}$ of Algorithm \ref{alg:threshold-spca} is a $(2p\delta + \eps)$-optimal solution for the first $K$ subproblems of the original SPCA problem for $(Q, p)$. That is, for each $k = 1, \dots, K$, we have
\[
u_k^\top Q u_k \ge v_k^* - (2p\delta + \eps),
\]
where $v_k^*$ is the optimal value of the $k$-th SPCA subproblem \eqref{eq:ospca} for $(Q, p)$.
\end{theorem}

\begin{proof}

By Theorem \ref{thm:threshold-decop}, it suffices to show that the vectors $\{\widetilde{z}_1,\dots,\widetilde{z}_K\}$ computed in Algorithm \ref{alg:threshold-spca} are the first $K$ components of an $\eps$-optimal solution to the SPCA problem for the block-diagonal matrix $A$. 

Observe that Algorithm \ref{alg:threshold-spca} maintains a set $S$ of candidate vectors, each being the next available sparse principal component from one block. At step $k$, the algorithm selects the vector $\widetilde{z}_k$ with the largest variance among $S$. This selection process is equivalent to sorting the union of all block-wise sparse principal components by their variances in descending order. Therefore, $\{\widetilde{z}_1,\dots,\widetilde{z}_K\}$ are exactly the first $K$ components of the sorted sequence of all block-wise components.

Since each block $A_i$ is solved independently using Algorithm~\ref{alg:ksparse} (or an $\eps$-optimal variant Algorithm~\ref{alg:branch-and-bound}), the collection of all block-wise components forms an $\eps$-optimal solution for $A$. Consequently, the first $K$ components of this sorted sequence constitute the first $K$ components of an $\eps$-optimal solution for $A$. Applying Theorem \ref{thm:threshold-decop} with the inverse permutation $\Pi^\top$ yields the desired result.
\end{proof}
Algorithm \ref{alg:threshold-spca} provides an efficient way to compute the first $K$ sparse principal components without solving for all $n$ components. The algorithm leverages the block-diagonal structure to decompose the problem and selects components greedily based on their variances. This approach is particularly beneficial when $K$ is small and the blocks are of moderate size.

\paragraph{Practical acceleration via Branch-and-Bound.}
In Algorithm \ref{alg:threshold-spca}, the call to Algorithm \ref{alg:ksparse} (line 6 and 13) can be replaced by its accelerated version using the branch-and-bound framework described in Appendix \ref{app:branch-and-bound}. This substitution is sound because the branch-and-bound Algorithm~\ref{alg:branch-and-bound} solves the same SPCA subproblem (for a single block $A_i$ under orthogonality constraints) to $\eps$-optimality. The theoretical guarantees of Theorem \ref{thm:threshold-alg-correctness} remain valid, as they depend only on the $\eps$-optimality of the block solutions, not on the specific method used to obtain them. 
This accelerated approach ensures that solutions with certified quality can be obtained within reasonable running time.

% Algorithm \ref{alg:ksparse} requires enumerating all $\binom{n}{p}$ support sets, which is computationally expensive for large $n$ and $p > 1$. However, in the case of the block-diagonal SPCA problem, each block $A_i$ has size $n_i$, which is typically much smaller than $n$. Therefore, the total cost $\sum_{i=1}^d \binom{n_i}{p_i}$ may be acceptable when $n_i$ and $p_i$ are small. This makes our overall approach feasible for large-scale problems that admit a block-diagonal approximation.
\section{Experimental Results}
In this section, we evaluate the numerical behavior of our proposed method.
Our goal is not to compete for the best empirical performance, but to verify that the algorithm is effective in enforcing orthogonality while maintaining a computational cost that is acceptable compared with non-orthogonal alternatives. The improvement in algorithm efficiency is evident from combining our proposed GS-SPCA algorithm with the decomposition framework and branch-and-bound method. Thus, we have temporarily omitted this part of the experiments in this version.
% Accordingly, we use a straightforward implementation and do not incorporate sophisticated engineering tricks or heavy hyper-parameter tuning.

\paragraph{Datasets.}
We conduct experiments using well-established benchmark datasets commonly employed in sparse feature selection and high-dimensional covariance analysis. Specifically, we utilize the \textsc{CovColon} dataset introduced by \cite{alon1999broad}. For our experiments, we selected the first 20 rows and 20 columns from this dataset.

\paragraph{Baselines.}
We compare against representative non-orthogonal SPCA approaches.
In particular, following \cite{berk2019certifiably}, we include the baseline that solves SPCA approximately via an adjusted (regularized) covariance matrix.

\paragraph{Notation / Preprocessing.}
For \textsc{CovColon}, we first apply Algorithm~\ref{alg:block-diag} to block-diagonalize the empirical covariance matrix.
We then select the largest diagonal block and use it as the input covariance matrix for all subsequent SPCA solvers.

% (To be filled.)
\begin{figure*}[htbp]
 \centering
    % 第一行的子图
    \subfloat[]{
        \includegraphics[width=0.3\textwidth]{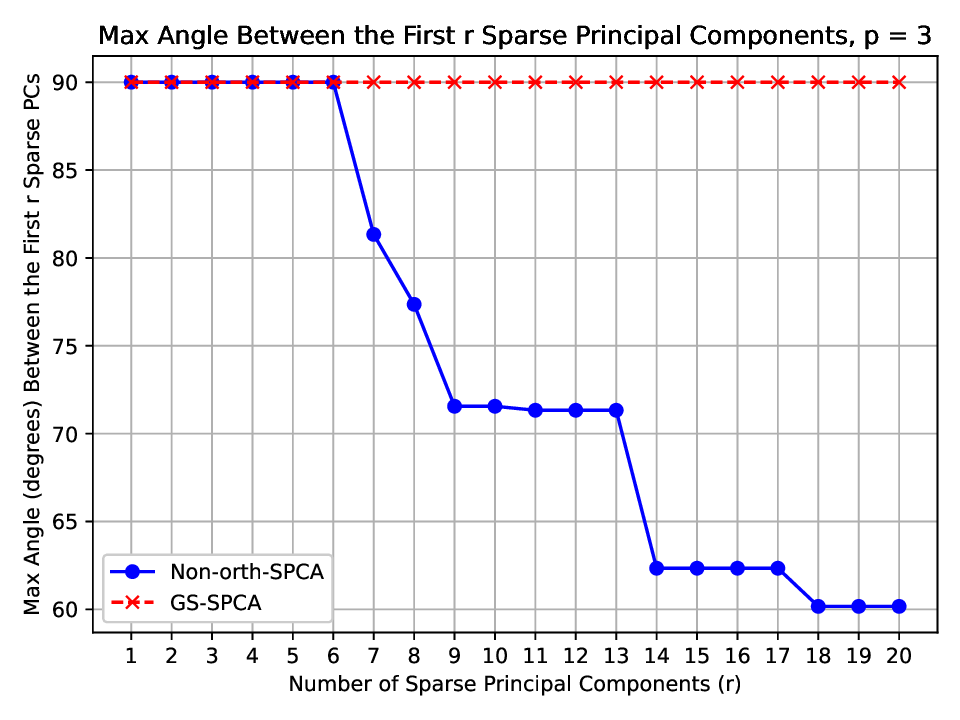}
        \label{fig:max_angle_k20_p3}
    }
    \subfloat[]{
        \includegraphics[width=0.3\textwidth]{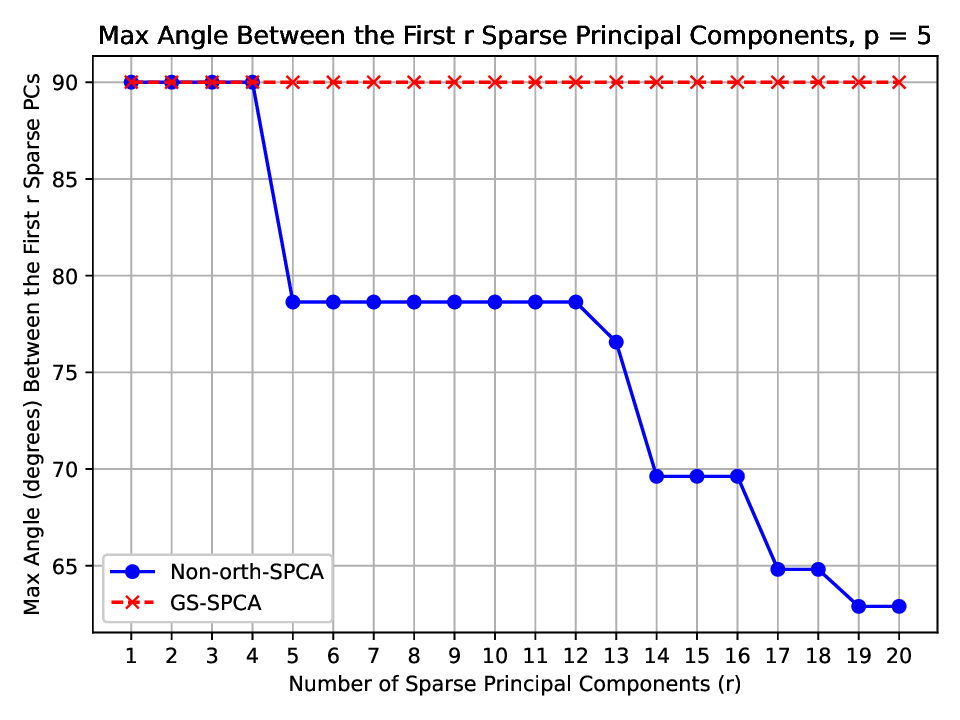}
        \label{fig:max_angle_k20_p5}
    }
    \subfloat[]{
        \includegraphics[width=0.3\textwidth]{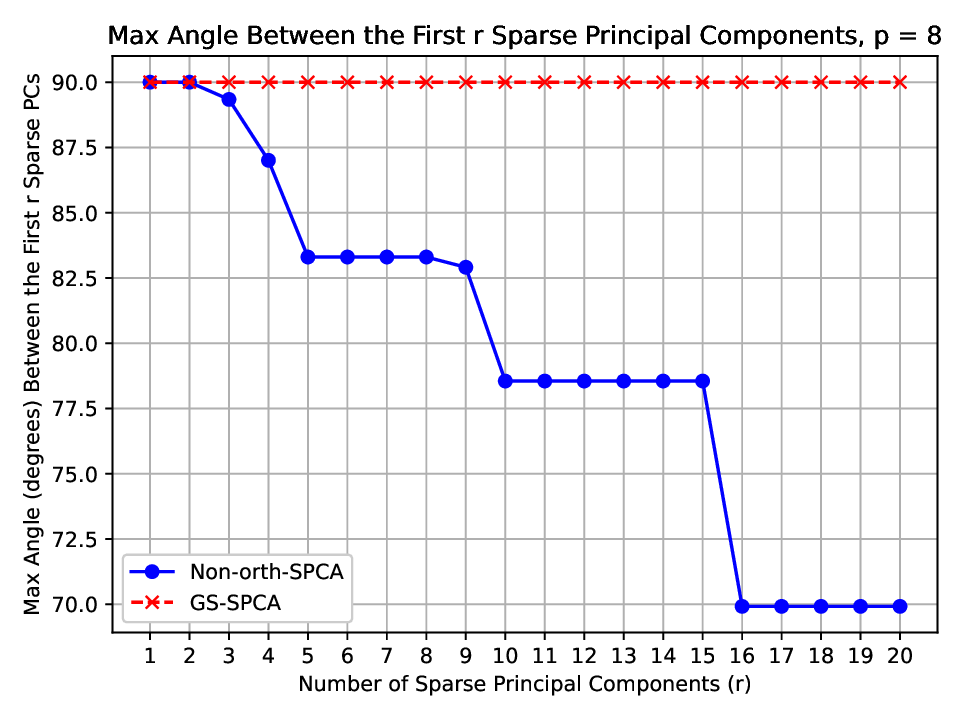}
        \label{fig:max_angle_k20_p8}
    }

    \vskip 10pt  % 控制子图之间的垂直间距

    % 第二行的子图
    \subfloat[]{
        \includegraphics[width=0.3\textwidth]{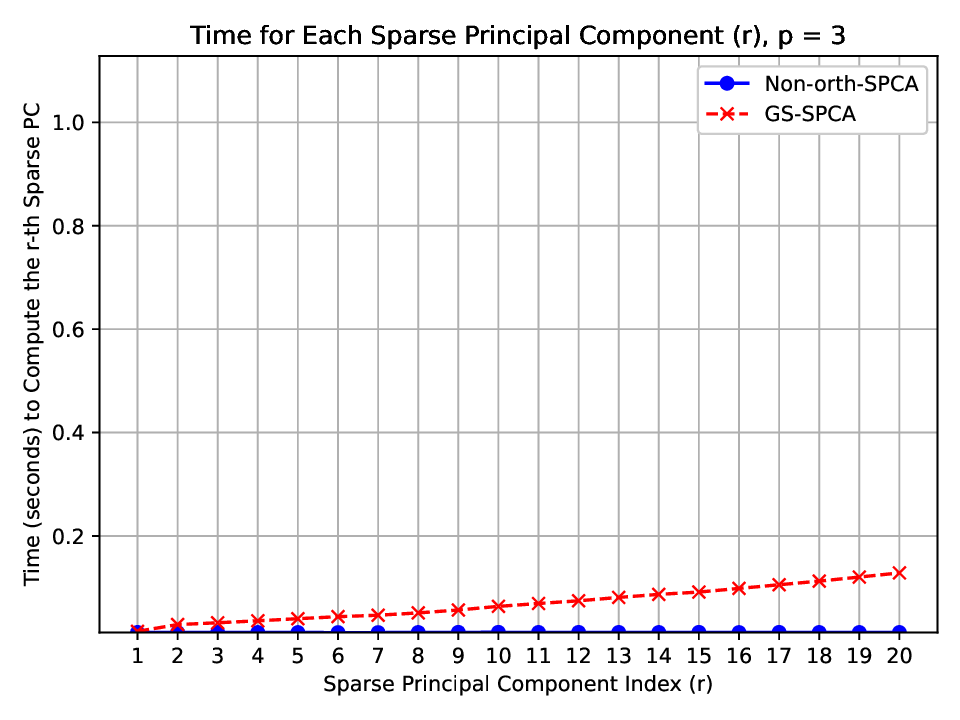}
        \label{fig:step_times_k20_p3}
    }
    \subfloat[]{
        \includegraphics[width=0.3\textwidth]{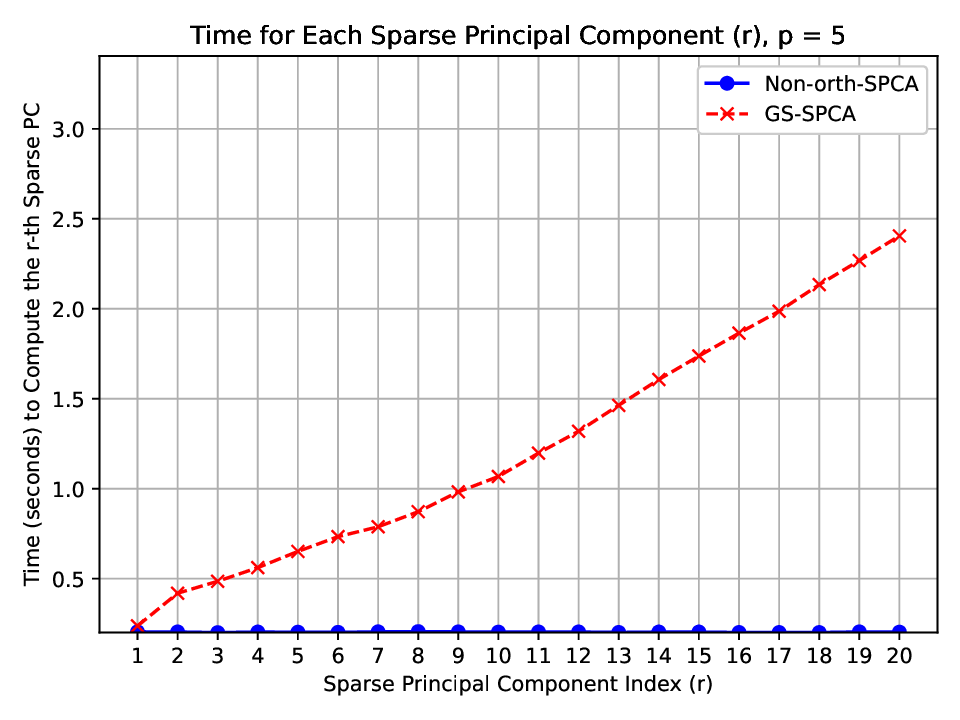}
        \label{fig:step_times_k20_p5}
    }
    \subfloat[]{
        \includegraphics[width=0.3\textwidth]{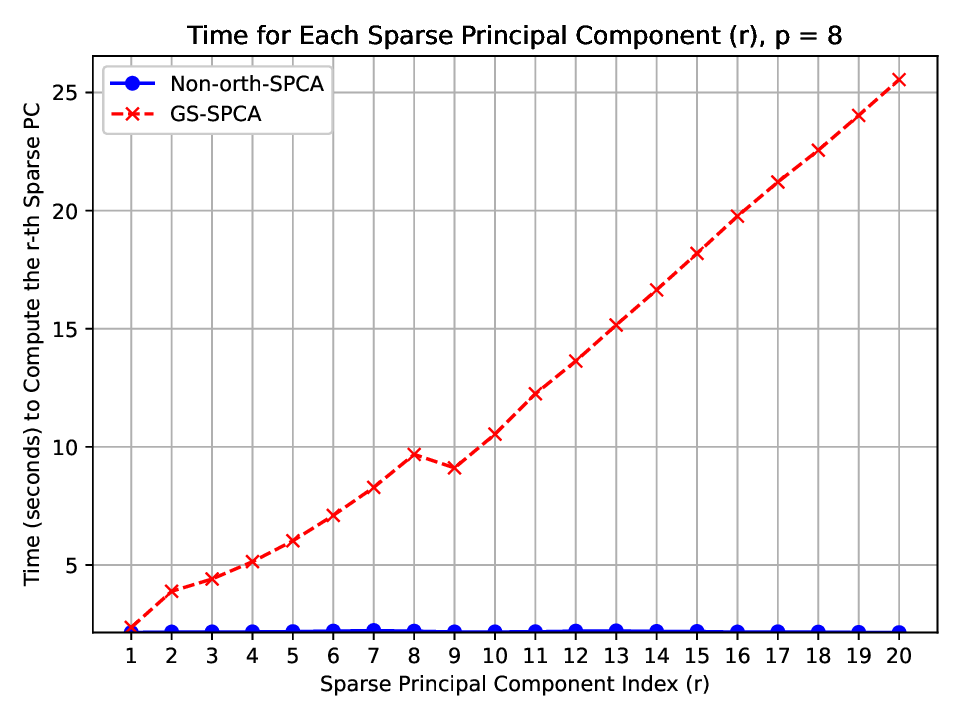}
        \label{fig:step_times_k20_p8}
    }

    \vskip 10pt  % 控制子图之间的垂直间距

    % 第三行的子图
    \subfloat[]{
        \includegraphics[width=0.3\textwidth]{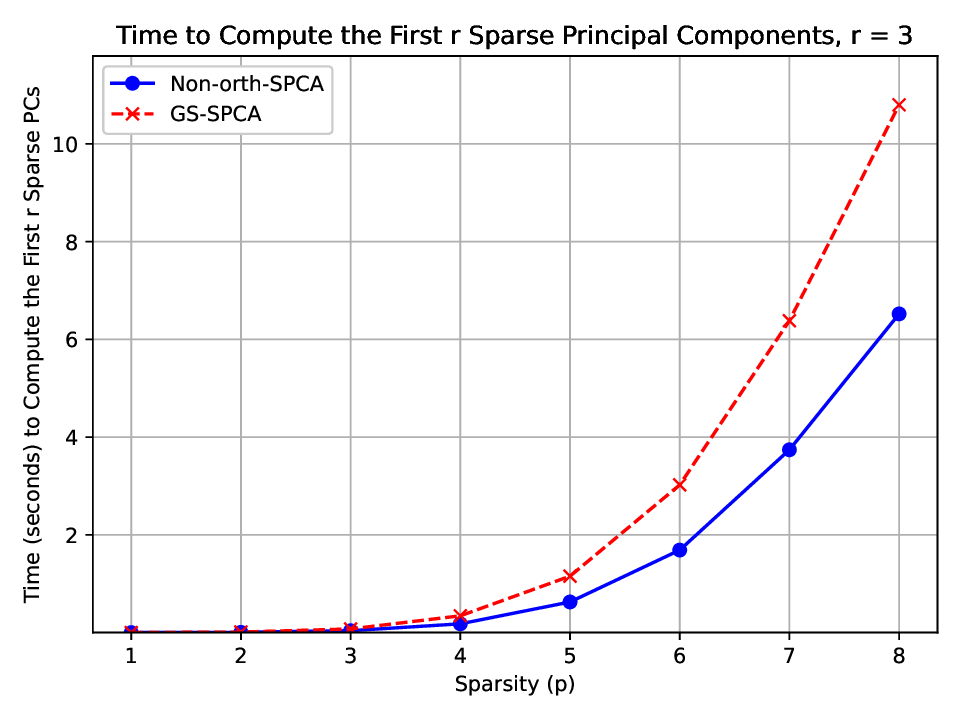}
        \label{fig:step_times_k20_r3}
    }
    \subfloat[]{
        \includegraphics[width=0.3\textwidth]{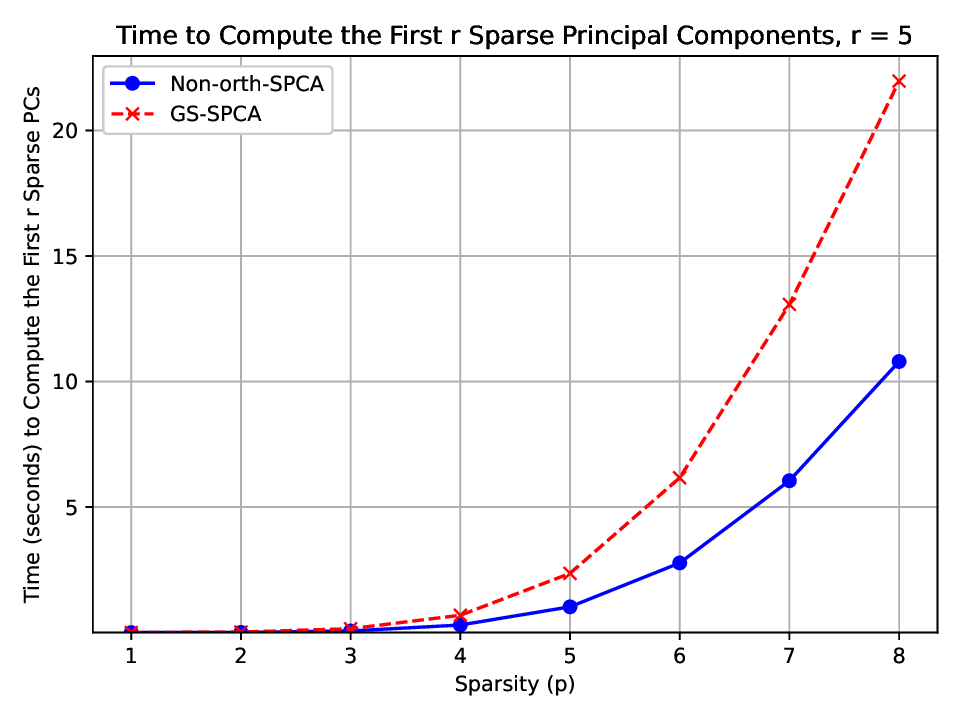}
        \label{fig:step_times_k20_r5}
    }
    \subfloat[]{
        \includegraphics[width=0.3\textwidth]{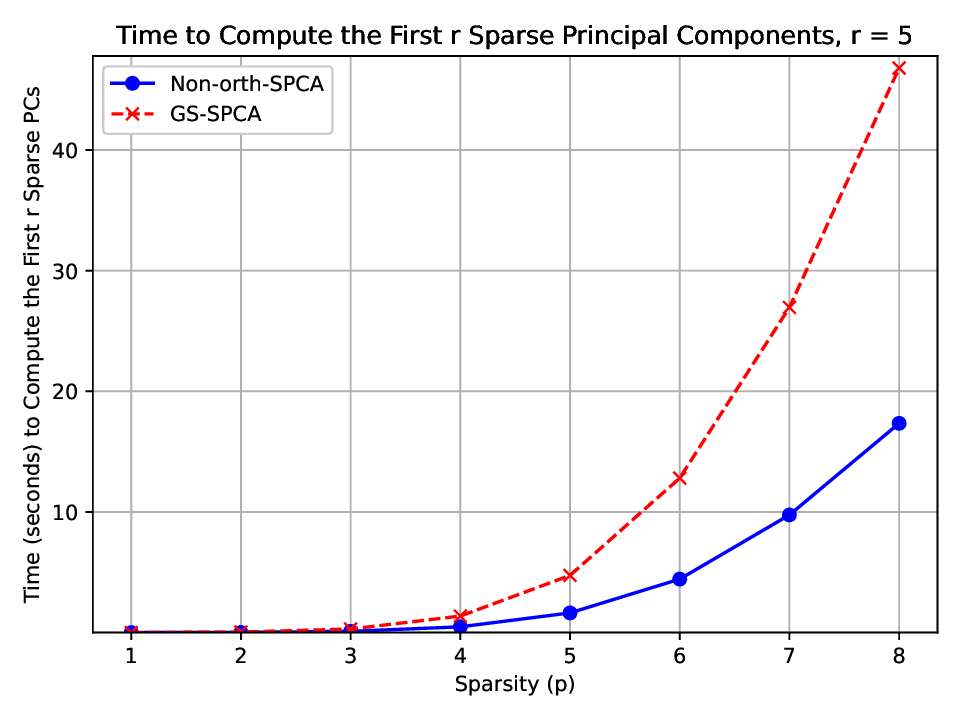}
        \label{fig:step_times_k20_r8}
    }

    \vskip 10pt  % 控制子图之间的垂直间距

    % 第四行的子图
    \subfloat[]{
        \includegraphics[width=0.3\textwidth]{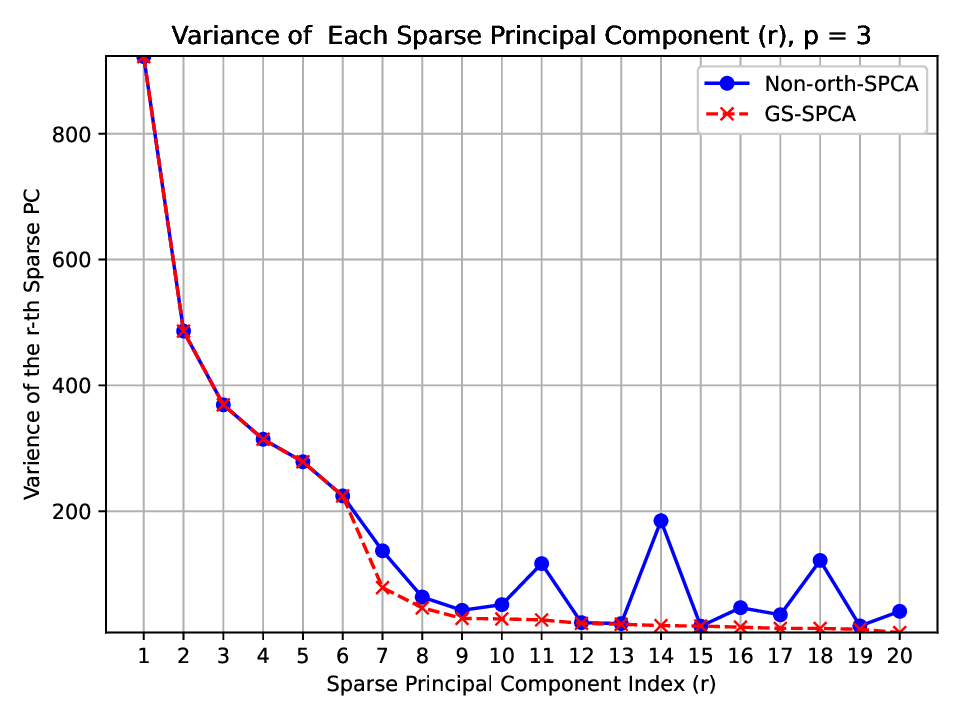}
        \label{fig:varience_k20_p3}
    }
    \subfloat[]{
        \includegraphics[width=0.3\textwidth]{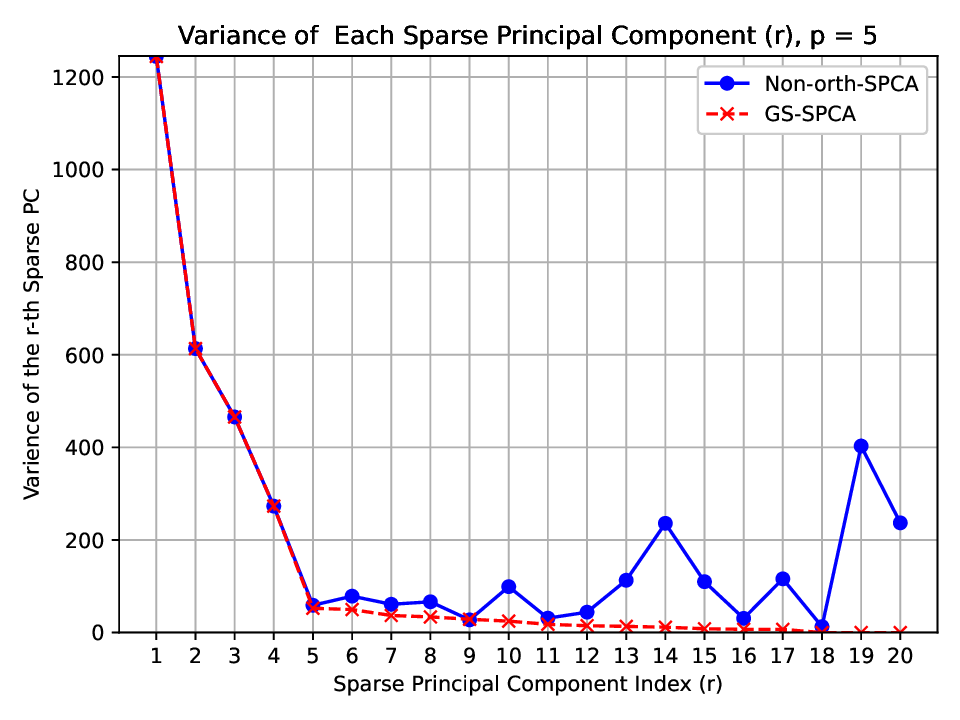}
        \label{fig:varience_k20_p5}
    }
    \subfloat[]{
        \includegraphics[width=0.3\textwidth]{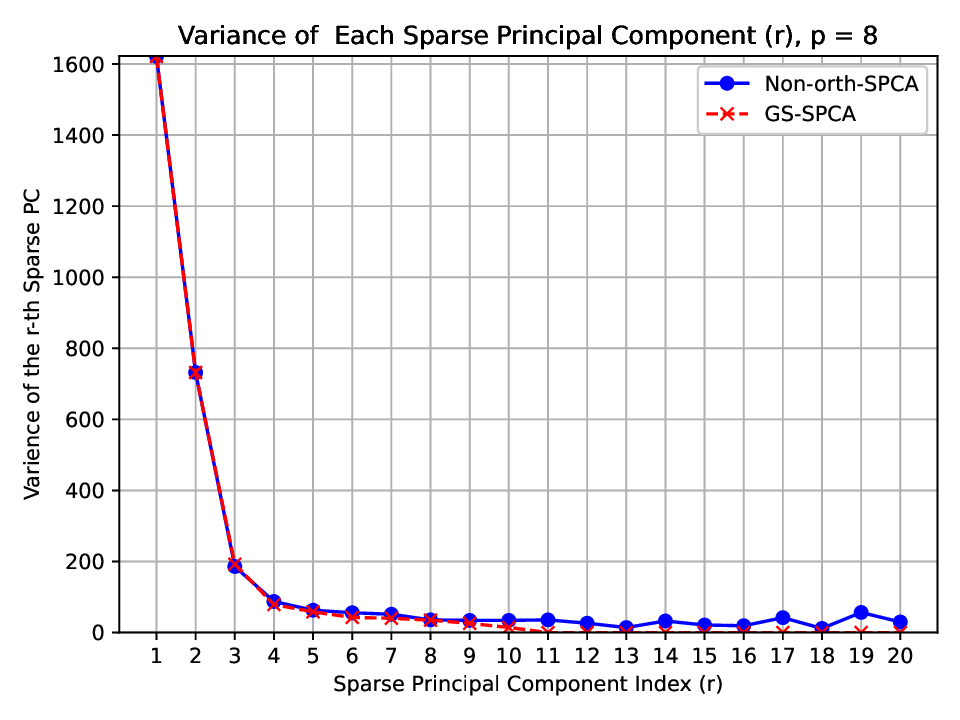}
        \label{fig:varience_k20_p8}
    }

    \caption{(a)-(c) Record the maximum angle between the first $r$ Sparse Principal Components; (d)-(f) Record the time to compute the $r$-th Sparse Principal Component for different sparsities, where each point represents the time to compute the $r$-th Sparse Principal Component; (g)-(i) Record the time for solving the first $r$ Sparse Principal Components as the sparsity changes; (j)-(l) Record the variance of the $r$-th Sparse Principal Component for different sparsities, where each point represents the variance of the $r$-th Sparse Principal Component.}
    \label{fig:comparison}
\end{figure*}
\paragraph{Results.}Figures \ref{fig:comparison} (a)-(c) show that, under different sparsities, the maximum angle between the first $r$ sparse principal components of the non-orthogonal SPCA algorithm increases as $r$ increases. This confirms the conclusion that the algorithm cannot guarantee orthogonality, as evidenced by the increasing angle between the sparse components.

Figures \ref{fig:comparison} (d)-(f) illustrate that, under different sparsities, the time required by our proposed GS-SPCA algorithm to compute the $r$-th sparse principal component increases with $r$. As $r$ increases, the number of vectors to be processed in the Gram-Schmidt orthogonalization also increases, leading to a longer time for the orthogonalization process. However, the increase is linear and remains within an acceptable range.

Figures \ref{fig:comparison} (g)-(i) show that when solving for the first $r$ sparse principal components, the computational time of the algorithm increases with increasing sparsity. In the GS-SPCA algorithm, this is due to the increase in the dimension of the vectors processed during Gram-Schmidt orthogonalization. Compared to the non-orthogonal algorithm, however, this increase is steady and acceptable.

Figures \ref{fig:comparison} (j)-(l) demonstrate that, under different sparsities, as $r$ increases, the variance of the $r$-th sparse principal component in the GS-SPCA algorithm gradually decreases to zero. In contrast, the non-orthogonal SPCA algorithm exhibits highly unstable variance decay due to the inability to guarantee orthogonality, making the variance reduction process much more erratic.
\section{Discussion for future works}\label{sec:future-work}
\textbf{Path Dependency of Variance in SPCA:} One of the fundamental differences between classical PCA and SPCA lies in the determinacy of the variance sequence $\{x_k^\top Q x_k\}_{k=1}^n$. In PCA, the variance $x_k^\top Q x_k$ equals the $k$-th largest eigenvalue of $Q$, a fixed value independent of the specific choice of preceding eigenvectors, even when eigenvalues have multiplicity.

In contrast, for SPCA, the variance $x_k^\top Q x_k$ depends on the choices of the preceding $k-1$ components. We illustrate this path dependency with a concrete example.

\textbf{Example:} Consider the following symmetric positive semi-definite covariance matrix
\[
Q = \begin{bmatrix}
5 & 1 & 0 \\
1 & 5 & 2 \\
0 & 2 & 2
\end{bmatrix},
\]
with sparsity parameter $p = 2$. The first sparse principal component problem \eqref{eq:spca1} has two different optimal solutions:
\[
x_1^{(a)} = \begin{bmatrix} \frac{1}{\sqrt{2}} & \frac{1}{\sqrt{2}} & 0 \end{bmatrix}^\top, \qquad
x_1^{(b)} = \begin{bmatrix} 0 & \frac{2}{\sqrt{5}} & \frac{1}{\sqrt{5}} \end{bmatrix}^\top,
\]
both achieving variance $6$. Depending on which $x_1$ is chosen, the subsequent orthogonal sparse principal components differ:

\begin{itemize}
    \item \textbf{Path A}: Choosing $x_1 = x_1^{(a)}$ yields the sequence
    \[
    \begin{aligned}
    x_1^{(a)} = \begin{bmatrix} \frac{1}{\sqrt{2}} & \frac{1}{\sqrt{2}} & 0 \end{bmatrix}^\top&: \text{variance} = 6, \\
    x_2^{(a)} = \begin{bmatrix} \frac{1}{\sqrt{2}} & -\frac{1}{\sqrt{2}} & 0 \end{bmatrix}^\top &: \text{variance} = 4, \\
    x_3^{(a)} = \begin{bmatrix} 0 & 0 & 1 \end{bmatrix}^\top &: \text{variance} = 2.
    \end{aligned}
    \]
    \item \textbf{Path B}: Choosing $x_1 = x_1^{(b)}$ yields the sequence
    \[
    \begin{aligned}
    x_1^{(b)} = \begin{bmatrix} 0 & \frac{2}{\sqrt{5}} & \frac{1}{\sqrt{5}} \end{bmatrix}^\top &: \text{variance} = 6, \\
    x_2^{(b)} = \begin{bmatrix} 1 & 0 & 0 \end{bmatrix}^\top &: \text{variance} = 5, \\
    x_3^{(b)} = \begin{bmatrix} 0 & \frac{1}{\sqrt{5}} & -\frac{2}{\sqrt{5}} \end{bmatrix}^\top &: \text{variance} = 1.
    \end{aligned}
    \]
\end{itemize}

Thus, both $\{x_1^{(a)},x_2^{(a)},x_3^{(a)}\}$ and $\{x_1^{(b)},x_2^{(b)},x_3^{(b)}\}$ are valid solutions to the SPCA problem under Definition \ref{def:ospca}, yielding different variance sequences $\{6,4,2\}$ and $\{6,5,1\}$. This directly demonstrates that the variance of the $k$-th sparse component is not predetermined but depends on choices of earlier $k-1$ components.
However, the sum of variances is always equal to $\tr(Q)$ as asserted by Proposition \ref{prop:tr_spca} (e.g., $\tr(Q)=12$ in this example). Thus, although sparsity allows the variance to be distributed flexibly among components, the total variance remains fixed.

In our proposed SPCA algorithm, the components are computed consecutively. In light of the
\textbf{Path Dependency of Variance in SPCA} illustrated above, this sequential design reveals
an important limitation: even if each step produces an optimal sparse component for the
corresponding single-component subproblem, different admissible choices of the early components
may induce \emph{different variance allocation paths} (variance sequences), and the resulting set
of sparse orthogonal components need not be \emph{jointly} optimal for the multivariate objective.
In other words, \emph{stage-wise optimality does not in general imply joint optimality}: when the
leading sparse component is non-unique, the subsequent feasible region can
change substantially, and the path dependency may accumulate into a non-negligible global
suboptimality, even though the total explained variance remains fixed at $\tr(Q)$.

Looking ahead, a key and highly impactful direction for future research is to extend our framework beyond the sequential paradigm and develop methods that target (exactly or approximately) the jointly optimal multivariate SPCA problem. This refers to the scenario where, in cases of two principal components with identical variances, we must select the solution corresponding to the second sparse principal component with the larger variance, i.e., how to identify path B in above Example. Addressing this challenge is crucial
both theoretically and practically: in many high-dimensional applications, what ultimately matters
is the \emph{collective} representational power and stability of the top $r$ sparse components, not
merely the variance captured by the first one. Consequently, designing scalable solvers for jointly
optimal SPCA is an important avenues for future work. Progress
along these lines would substantially strengthen the reliability and reproducibility of SPCA in
multi-component settings, and push the performance ceiling of sparse, interpretable principal
subspace estimation.
%\cite{journee2010generalized}

% From an implementation perspective, there also remains substantial room for engineering and acceleration.blablalblablalbla

\bibliography{example_paper}
\bibliographystyle{unsrt}

\newpage
\appendix
\onecolumn

\section{Formal Proof of Decomposition Theorems}

\subsection{Proof of Theorem \ref{thm:block-decomp}}
\label{app:proof-decop}
\begin{proof}
We need to verify that the constructed sequence $\{z_k\}_{k=1}^n$ meets all conditions of Definition \ref{def:ospca} for $(Q,p)$.

\textbf{Orthonormality and sparsity:} 
For each $1\le i\le d$, the vectors $\{y_j^{(Q_i)}\}_{j=1}^{n_i}$ form an orthonormal basis of $\mathbb{R}^{n_i}$ by Definition \ref{def:ospca}. Since the extended vectors $\{y_j^{(\W_i)}\}_{j=1}^{n_i}$ belong to mutually orthogonal subspaces $\W_i$ for different $i$, all of vectors $\bigcup_{i=1}^d \{y_j^{(\W_i)}\}_{j=1}^{n_i}$ are orthonormal in $\mathbb{R}^n$. Sorting does not affect orthonormality and sparsity, so each $z_k$ satisfies $\|z_k\|_2 = 1$ and $\|z_k\|_0 \leq \max_i p_i \leq p$.

\textbf{Optimality of the first component:}
For any feasible $z \in \mathbb{R}^n$ with $\|z\|_2=1$ and $\|z\|_0 \le p$, decompose it as $z = \sum_{i=1}^d z^{(\W_i)}$ where $z^{(\W_i)} \in \W_i$. Then
\begin{align*}
z^\top Q z &= \sum_{i=1}^d (z^{(\W_i)})^\top Q z^{(\W_i)} = \sum_{i=1}^d (z^{(Q_i)})^\top Q_i z^{(Q_i)} \\
&\le \sum_{i=1}^d \left( \|z^{(Q_i)}\|_2^2 \max_{
\|y\|_2=1,~\|y\|_0\le p_i} y^\top Q_i y \right)\\
&= \sum_{i=1}^d \left( \|z^{(Q_i)}\|_2^2 \cdot (y_1^{(Q_i)})^\top Q_i y_1^{(Q_i)} \right)
\end{align*}
Note that for all $1\le i\le d$,
\[
(y_1^{(Q_i)})^\top Q_i y_1^{(Q_i)} =(y_1^{(\W_i)})^\top Q y_1^{(\W_i)} \le z_1^\top Q z_1,
\]
and the Parseval identity
\[
\sum_{i=1}^d \|z^{(Q_i)}\|_2^2 = \|z\|_2^2 = 1,
\]
we have $z^\top Q z \le z_1^\top Q z_1$. Hence, $z_1$ is a valid first sparse principal component for $Q$.

\textbf{Inductive step for subsequent components:}
Assume by induction that $z_1, \dots, z_{k-1}$ satisfy Definition \ref{def:ospca}.
For each block $i$, let $m_i$ denote the number of vectors among $z_1, \dots, z_{k-1}$ that lie in the subspace $\W_i$:
\[
m_i:=| \{\ell:1 \le \ell \le k-1,\;z_\ell\in\W_i\} |.
\]
Thus, the set $\{z_1, \dots, z_{k-1}\}$ can be represented in the form of the following disjoint union:
\[
\{z_1, \dots, z_{k-1}\}= \bigsqcup_{i=1}^d \{y_1^{(\W_i)},\dots,y_{m_i}^{(\W_i)}\}
\]

Let $z \in \mathbb{R}^n$ be any feasible vector for the $k$-th SPCA subproblem with the original $(Q,p)$ (i.e., $\|z\|_2 = 1$, $\|z\|_0 \le p$, and $z \perp z_1, \dots, z_{k-1}$).
Decompose $z$ as $z = \sum_{i=1}^d z^{(\W_i)}$ with $z^{(\W_i)} \in \W_i$. Then for all $1\le i\le d$ and $1\le\ell\le m_i$,
\[
0=\langle z, y_\ell^{(\W_i)}\rangle = \langle z^{(\W_i)}, y_\ell^{(\W_i)}\rangle = \langle z^{(Q_i)}, y_\ell^{(Q_i)}\rangle,
\]
thus, we have
\begin{align*}
z^\top Q z &= \sum_{i=1}^d (z^{(\W_i)})^\top Q z^{(\W_i)} = \sum_{i=1}^d (z^{(Q_i)})^\top Q_i z^{(Q_i)} \\
&\le \sum_{i=1}^d \left( \|z^{(Q_i)}\|_2^2 \max_{
\begin{aligned}
\|y\|_2=1,~\|y\|_0\le p_i,\\
y \perp y_1^{(Q_i)},\dots,y_{m_i}^{(Q_i)}
\end{aligned}} 
y^\top Q_i y 
\right)\\
&\le \sum_{i=1}^d \left( \|z^{(Q_i)}\|_2^2 \cdot (y^{(Q_i)}_{m_i+1})^\top Q_i y^{(Q_i)}_{m_i+1}\right) \\
&\le \sum_{i=1}^d \left( \|z^{(Q_i)}\|_2^2 \cdot z_k^\top Q z_k\right) = z_k^\top Q z_k
\end{align*}
Therefore, $z_k$ is a solution for the $k$-th original SPCA subproblem with $(Q,p)$. By induction, the entire sequence satisfies Definition \ref{def:ospca}, completing the proof.
\end{proof}

\subsection{Proof of Theorem \ref{thm:eps-block-decomp}}
\label{app:proof-decop-eps}
\begin{proof}
We follow a similar structure to the proof of Theorem \ref{thm:block-decomp}, focusing on the changes due to the $\eps$-optimality.

\textbf{Orthonormality and sparsity:} As before, the extended vectors form an orthonormal system, and each $\widetilde{z}_k$ satisfies $\|\widetilde{z}_k\|_2=1$ and $\|\widetilde{z}_k\|_0 \le p$.

\textbf{Optimality of the first component:}
Let 
$$v_1^* = \max_{\|x\|_2=1,\; \|x\|_0\le p} x^\top Q x$$
be the optimal value for the first sparse principal component of $Q$. For any block $i$, let 
$$v_{1,i}^* = \max_{\|x\|_2=1,\; \|x\|_0\le p_i} x^\top Q_i x.$$
Let $i^* = \arg\max_i v_{1,i}^*$, we have $v_1^* = v_{1,i^*}$ by Theorem \ref{thm:block-decomp}.
Since $\{\widetilde{y}_j^{(Q_{i^*})}\}$ is an $\eps$-optimal solution for $Q_{i^*}$, we have
\[
(\widetilde{y}_1^{(Q_{i^*})})^\top Q_{i^*} \widetilde{y}_1^{(Q_{i^*})} \ge v_{1,i^*}^* - \eps = v_1^* - \eps.
\]
Moreover, $\widetilde{z}_1$ is the extended vector with maximum variance among all $\widetilde{y}_1^{(\W_i)}$, so
\[
\widetilde{z}_1^\top Q \widetilde{z}_1 \ge (\widetilde{y}_1^{(Q_{i^*})})^\top Q_{i^*} \widetilde{y}_1^{(Q_{i^*})} \ge v_1^* - \eps.
\]

\textbf{Inductive step:}
Assume by induction that for $k \ge 2$, the vectors $\widetilde{z}_1, \dots, \widetilde{z}_{k-1}$ satisfy the $\eps$-optimality condition up to step $k-1$.

Let
\[ v_k^* =\max_{\substack{\|x\|_2=1,~\|x\|_0 \le p \\ x \perp \widetilde{z}_1, \dots, \widetilde{z}_{k-1}}} x^\top Q x.
\]
For each block $i$, let $m_i$ be the number of vectors among $\widetilde{z}_1,\dots,\widetilde{z}_{k-1}$ that lie in $\W_i$. Define
\[ v_{k,i}^* =\max_{\substack{\|x\|_2=1,~\|x\|_0 \le p \\ x \perp \widetilde{y}_1^{(Q_i)}, \dots, \widetilde{y}_{m_i}^{(Q_i)}}} x^\top Q_i x.
\]
Consider the block-wise decomposition as in Theorem \ref{thm:block-decomp}, we have $v_k^* = \max_i v_{k,i}^* = v_{k,i^*}$ for some block $i^*$.
Since $\{\widetilde{y}_j^{(Q_{i^*})}\}_{j=1}^{n
_{i^*}}$ is an $\eps$-optimal solution for $Q_{i^*}$, we have
\[
(\widetilde{y}_{m_{i^*}+1}^{(Q_{i^*})})^\top Q_{i^*} \widetilde{y}_{m_{i^*}+1}^{(Q_{i^*})} \ge v_{k,i^*}^* - \eps = v_k^* - \eps.
\]
Therefore,
\[
\widetilde{z}_k^\top Q \widetilde{z}_k \ge (\widetilde{y}_{m_{i^*}+1}^{(Q_{i^*})})^\top Q_{i^*} \widetilde{y}_{m_{i^*}+1}^{(Q_{i^*})} \ge v_k^* - \eps.
\]

This completes the induction, showing that $\{\widetilde{z}_1, \dots, \widetilde{z}_n\}$ is an $\eps$-optimal solution.
\end{proof}

\subsection{Proof of Theorem \ref{thm:threshold-decop}}
\label{app:proof-approx}.
\begin{proof}
Let
\[
u_k = \Pi^\top \widetilde{z}_k,\;\;\forall 1\le k\le n.
\]
By Definition \ref{def:eps-spca}, for each $k=1,\dots,n$,
\begin{align*}
    \widetilde{z}_k^\top A \widetilde{z}_k &\ge \max_{\substack{\|z\|_2=1,~\|z\|_0 \le p \\ z \perp \widetilde{z}_1, \dots, \widetilde{z}_{k-1}}} z^\top A z - \eps \\
    &= \max_{\substack{\|\Pi^\top z\|_2=1,~\|\Pi^\top z\|_0 \le p \\ z \perp \widetilde{z}_1, \dots, \widetilde{z}_{k-1}}} (\Pi^\top z)^\top Q^\delta (\Pi^\top z) - \eps \\
    &= \max_{\substack{\|u\|_2=1,~\|u\|_0 \le p \\ u \perp u_1, \dots, u_{k-1}}} u^\top Q^\delta u - \eps.
\end{align*}
Since $\|Q - Q^\delta\|_{\max} \leq \delta$ and the sparsity constraint $\|u\|_0\le p$, by Cauchy-Schwarz inequality, we have for any $p$-sparse unit vector $u$:
\[
|u^\top (Q-Q^\delta) u| \le p\delta.
\]
Therefore, for each $k=1,\dots,n$,
\[
\widetilde{z}_k^\top A \widetilde{z}_k \ge \max_{\substack{\|u\|_2=1,~\|u\|_0 \le p \\ u \perp u_1, \dots, u_{k-1}}} u^\top Q u - p\delta - \eps.
\]
Together with $
u_k^\top Q u_k + p\delta \ge u_k^\top Q^\delta u_k = \widetilde{z}_k^\top A \widetilde{z}_k$,
we have
\[
u_k^\top Q u_k \ge \max_{\substack{\|u\|_2=1,~\|u\|_0 \le p \\ u \perp u_1, \dots, u_{k-1}}} u^\top Q u - (2p\delta + \eps),
\]
which means $\{u_k\}_{k=1}^{n}$ is a $(2p\delta+\eps)$-optimal solution to the SPCA problem for original $(Q,p)$.
\end{proof}

Note that the threshold matrix $A$ (and $Q^\delta$) may not be positive semi-definite. However, the SPCA problem (Definition \ref{def:ospca}) and the $\eps$-optimality criterion (Definition \ref{def:eps-spca}) are well-defined for any symmetric matrix, as the maximization is over a compact set. The proof only uses the symmetry of $A$ and the entrywise bound $\|Q-Q^\delta\|_{\max}\leq\delta$, not positive semi-definiteness. Therefore, the theorem statement and proof remain valid without requiring $A$ to be positive semi-definite.

\section{Accelerative algorithm for an $\eps$-optimal solution of SPCA problem} \label{app:branch-and-bound}

As introduced in \cite{berk2019certifiably}, we consider a family of subproblems called SPCA with Partially Determined Support (SPCA-PDS). Each such problem is defined by lower and upper bound vectors $\mathbf{l}, \mathbf{u} \in \{0,1\}^n$ that restrict the support vector $y$ in the $k$-th SPCA-MIO formulation \eqref{eq:spca-mio}:
\begin{equation}
\label{eq:spca-pds}
\begin{aligned}
\max \quad & x^\top Q x, \\
\mathrm{s.t.}  \quad & x \perp x_1,\dots,x_{k-1},\\
& \|x\|_2 = 1, \\
& \|y\|_0 = p,\\
& -y[j]\le x[j] \le y[j],\;\; 1\le j\le n,\;\;\;\; \text{(SPCA-PDS)}\\
& x\in[-1,1]^n,\\
& y\in\{0,1\}^n,\\
& \mathbf{l}\le y \le \mathbf{u},\;\; \mathbf{l}, \mathbf{u}\in\{0,1\}^n.
\end{aligned}
\end{equation}
Let $X_k = [x_1,\dots,x_{k-1}]$. The set of feasible vectors $x$ for problem \eqref{eq:spca-pds} is denoted by
\[
\mathbb{X}_k(\mathbf{l},\mathbf{u},p):=\{x\in\mathbb{R}^n: X_k^\top x=0,\ \|x\|_2=1,\ \|x\|_0 \le p,\ \mathbf{l} \le x \le \mathbf{u}\}.
\]

The branch-and-bound algorithm (Algorithm \ref{alg:branch-and-bound}) relies on two key functions that compute lower and upper bounds for SPCA-PDS:
\begin{align}
\text{lower\_bound}(Q, p, X_k, \mathbf{l}, \mathbf{u}) &=  \left(x^\top Q x \mathrm{~for~some~}x\in\mathbb{X}_k(\mathbf{l},\mathbf{u},p)\right) \label{eq:low-bound}\\
\text{upper\_bound}(Q, p, X_k, \mathbf{l}, \mathbf{u}) &\ge  \left(x^\top Q x \mathrm{~for~all~}x\in\mathbb{X}_k(\mathbf{l},\mathbf{u},p)\right).\label{eq:up-bound}
\end{align}
These bounds \eqref{eq:low-bound} and \eqref{eq:up-bound} guide the search by pruning nodes whose upper bound does not exceed the current best lower bound by at least the tolerance $\epsilon \ge 0$.

\begin{algorithm}[t]
\caption{Branch-and-Bound Algorithm for the $k$-th SPCA Problem}
\label{alg:branch-and-bound}
\begin{algorithmic}[1]
\REQUIRE Symmetric matrix $Q \in \mathbb{R}^{n \times n}$, sparsity $p$, a matrix formed by previous components $X_k=[x_1,\dots,x_{k-1}]$ (empty if $k=1$), optimality tolerance $\epsilon \geq 0$.
\ENSURE A $k$-th sparse principal component $x_k$ that is globally optimal within tolerance $\epsilon$.

\STATE Initialize root node $n_0 = (\mathbf{l}, \mathbf{u}) = (\{0\}^n, \{1\}^n)$
\STATE Initialize node set $\mathcal{N} = \{n_0\}$
\STATE Initialize $x_k^* = \{0\}^n$, 
%\STATE Compute $ub(n_0) = \text{upper\_bound}(Q, p, X, \mathbf{l}, \mathbf{u})$
\STATE Set $LB = -\infty$, $UB = \lambda_{\max}((I-X_kX_k^\top)Q(I-X_kX_k^\top))$

\WHILE{$UB - LB > \epsilon$}
    \STATE Select a node $(\mathbf{l}, \mathbf{u}) \in \mathcal{N}$ %(e.g., with maximal $ub$)
    \STATE select some index $i$ where $\mathbf{l}[i]=0$, $\mathbf{u}[i]=1$
    \FOR{$val = 0, 1$}
        \STATE Create new node: 
        \STATE \quad $\mathbf{l}' \gets (l_1, \dots, l_{i-1}, val, l_{i+1}, \dots, l_n),\;\;\mathbf{u}' \gets (u_1, \dots, u_{i-1}, val, u_{i+1}, \dots, u_n)$
        \IF{$\|\mathbf{l}'\|_0 > p$ \OR $\|\mathbf{u}'\|_0 < p$}
            \STATE \textbf{continue} \COMMENT{Prune by sparsity constraint}
        \ENDIF
        \IF{$\|\mathbf{l}'\|_0  = p$ \OR $\|\mathbf{u}'\|_0=p$}
            \STATE Compute the optimal solution $x$ to the problem at $(\mathbf{l}',\mathbf{u}')$ with the determined support set $Y$
            \STATE Set $lb=ub=\lambda_{\max}(P_YQ_YP_Y)$, where the matrix $P_Y$ can be obtained by Line 3--5 in Algorithm \ref{alg:ksparse}.
        \ELSE 
        \STATE Compute $lb = \text{lower\_bound}(Q, p, X_k, \mathbf{l}', \mathbf{u}')$, with corresponding vector $x$  in the feasibility set $\mathbb{X}_k(\mathbf{l}',\mathbf{u}',p)$
        \STATE Compute $ub = \text{upper\_bound}(Q, p, X_k, \mathbf{l}', \mathbf{u}')$
        \ENDIF
        \IF{$lb > LB$}
            \STATE Update $LB \gets lb$, $x_k^* \gets x$
            \STATE Remove any node with $up \le LB$ from $\mathcal{N}$
        \ENDIF
        \IF{$ub > LB$}
            \STATE Add $(\mathbf{l}', \mathbf{u}')$ to $\mathcal{N}$
        \ENDIF
    \ENDFOR
    
    \STATE Remove $(\mathbf{l}, \mathbf{u})$ from $\mathcal{N}$
    \STATE Update $UB$ to be the greatest value of $ub$ over $\mathcal{N}$
\ENDWHILE

\STATE \textbf{Return} $x_k^*$
\end{algorithmic}
\end{algorithm}

It should be noted that the upper bound function $upper(\mathbf{l},\mathbf{u},p)$ of \cite{berk2019certifiably} is defined as
\[
upper(\mathbf{l},\mathbf{u},p) \ge \left(x^\top Q_k x \mathrm{~for~all~}x\in\widetilde{\mathbb{X}}(\mathbf{l},\mathbf{u},p)\right),
\]
with adjusted matrix $Q_k=(I-X_kX_k^\top)Q(I-X_kX_k^\top)$ and feasible set 
\[
\widetilde{\mathbb{X}}(\mathbf{l},\mathbf{u},p)=\{x\in\mathbb{R}^n: \|x\|_2=1,\ \|x\|_0 \le p,\ \mathbf{l} \le x \le \mathbf{u}\}.
\]
Since $\mathbb{X}_k(\mathbf{l},\mathbf{u},p) \subseteq \widetilde{\mathbb{X}}(\mathbf{l},\mathbf{u},p)$, we have
\begin{equation*}
upper(\mathbf{l},\mathbf{u},p) \ge \left(x^\top Q_k x \mathrm{~for~all~}x\in\widetilde{\mathbb{X}}(\mathbf{l},\mathbf{u},p)\right) 
\ge \left(x^\top Q_k x \mathrm{~for~all~}x\in\mathbb{X}_k(\mathbf{l},\mathbf{u},p)\right)
= \left(x^\top Q x \mathrm{~for~all~}x\in\mathbb{X}_k(\mathbf{l},\mathbf{u},p)\right).
\end{equation*}
Thus the function $upper(\mathbf{l},\mathbf{u},p)$ from \cite{berk2019certifiably} also serves as a valid $\text{upper\_bound}(Q, p, X_k, \mathbf{l}, \mathbf{u})$ in our algorithm.

For the lower bound, it suffices to produce any feasible vector $x\in\mathbb{X}_k(\mathbf{l},\mathbf{u},p)$. We can first obtain a feasible vector $x\in\widetilde{\mathbb{X}}(\mathbf{l},\mathbf{u},p)$ using the same method as in \cite{berk2019certifiably}, then generate another vector $\hat{x}\in\mathbb{X}_k(\mathbf{l},\mathbf{u},p)$ close to $x$. Specifically, let $Y = \operatorname{supp}(x)$ be the support of $x$. We then restrict the previous $k-1$ components to $Y$, obtaining $v_i = x_i[Y]$ for $i=1,\dots,k-1$. Next, we compute an orthonormal basis $U_Y$ for the subspace spanned by $\{v_1,\dots,v_{k-1}\}$ via the Gram--Schmidt process. Finally, we project $x[Y]$ onto the orthogonal complement of this subspace: $z = (I - U_YU_Y^\top) x[Y]$, normalize it to unit vector $\hat{z} = z/\|z\|_2$, and extend $\hat{z}$ to an $n$-dimensional vector by setting entries outside $Y$ to zero. The resulting vector, denoted by $\hat{x}$, satisfies $\hat{x} \perp x_1,\dots,x_{k-1}$, $\|\hat{x}\|_2=1$ and $\|\hat{x}\|_0 \le p$, hence $\hat{x} \in \mathbb{X}_k(\mathbf{l},\mathbf{u},p)$. The variance $\hat{x}^\top Q \hat{x}$ then provides a valid lower bound.

\begin{theorem}[Correctness of Algorithm~\ref{alg:branch-and-bound}] Algorithm~\ref{alg:branch-and-bound} terminates in finitely
many iterations at an $\eps$-optimal solution for the $k$-th SPCA-MIO problem \eqref{eq:spca-mio}.
\end{theorem}
\begin{proof}
The proof is the same as Theorem 1 of \cite{berk2019certifiably}.
\end{proof}

\end{document}